\documentclass[twoside]{article}

\usepackage[accepted]{aistats2024}
\usepackage{svg}
\usepackage{microtype}
\usepackage{graphicx}
\usepackage{booktabs} % for professional tables
\usepackage{subcaption}

\usepackage{hyperref}
\usepackage{amsmath,amssymb,amsfonts}
\usepackage{algorithmic}
\usepackage{algorithm}
\usepackage{graphicx}
\usepackage{textcomp}
\usepackage{xcolor}
\usepackage[numberedsection=false]{glossaries}
\usepackage{amsthm}
\usepackage{dblfloatfix} 
\usepackage{soul}
\usepackage{bm}
\DeclareMathAlphabet{\pazocal}{OMS}{zplm}{m}{n}

\newcommand{\mathify}[1]{\ifmmode{#1}\else\mbox{$#1$}\fi}

\usepackage{tikz}
\usepackage{pstricks}
\usetikzlibrary{arrows,shapes}
\usetikzlibrary{fit}

\newcommand{\bea}{\begin{eqnarray}}
\newcommand{\eea}{\end{eqnarray}}
\newcommand{\beas}{\begin{eqnarray*}}
\newcommand{\eeas}{\end{eqnarray*}}
\usepackage{amsfonts}

\newcommand{\suchthat}{\;\ifnum\currentgrouptype=16 \middle\fi|\;}
\newcommand{\dbar}{\;\ifnum\currentgrouptype=16 \middle \fi\|\;}

\usepackage{amsmath}
\usepackage{amssymb}
\usepackage{mathtools}
\usepackage{amsthm}

\allowdisplaybreaks

\usepackage[capitalize,noabbrev]{cleveref}

\theoremstyle{plain}
\newtheorem{theorem}{Theorem}[section]

\theoremstyle{definition}

\theoremstyle{remark}

\theoremstyle{conjecture}

\usepackage[textsize=tiny]{todonotes}

\newacronym{iot}{IoT}{Internet of Things}
\newacronym{ml}{ML}{machine learning}
\newacronym{dl}{DL}{Deep Learning}
\newacronym{marl}{MARL}{multi-agent reinforcement learning}
\newacronym{rl}{RL}{reinforcement learning}
\newacronym{decpomdp}{Dec-POMDP}{Decentralized Partially Observable Markov Decision Process }
\newacronym{pomdp}{POMDP}{Partially Observable Markov Decision Process}
\newacronym{uav}{UAV}{Unmanned Aerial Vehicle}
\newacronym{dqn}{DQN}{Deep Q-Network}
\newacronym{dnn}{DNN}{deep neural network}
\newacronym{dial}{DIAL}{Differentiable Inter-Agent Learning}
\newacronym{mdp}{MDP}{Markov decision process}
\newacronym{fov}{FoV}{Field of View}
\newacronym{cnn}{CNN}{Convolutional Neural Network}
\newacronym{nn}{NN}{neural network}
\newacronym{ddql}{DDQL}{Distributed Deep Q-Learning}
\newacronym{pdf}{PDF}{Probability Density Function}
\newacronym{ndpomdp}{ND-POMDP}{Networked Distributed Partially Observable Markov Decision Process}
\newacronym{radam}{RAdam}{Rectified Adam}
\newacronym{cdf}{CDF}{cumulative distribution function}
\newacronym{mpc}{MPC}{Model Predictive Control}
\newacronym{rv}{rv}{Random Variable}
\newacronym{qoe}{QoE}{Quality of Experience}
\newacronym{tlc}{TLC}{Telecommunications}
\newacronym{cml}{CML}{communications for machine learning}
\newacronym{mlc}{MLC}{machine learning for communications}
\newacronym{drl}{DRL}{deep reinforcement rearning}
\newacronym{rf}{RF}{Radio Frequency}
\newacronym{urllc}{URLLC}{Ultra-Reliable and Low-Latency Communications}
\newacronym{fl}{FL}{federated learning}
\newacronym{kpi}{KPI}{Key Performance Indicators}
\newacronym{mec}{MEC}{Mobile Edge Computing}
\newacronym{ei}{EI}{Edge Intelligence}
\newacronym{bs}{BS}{base station}
\newacronym{sdn}{SDN}{Software Defined Networking}
\newacronym{mimo}{MIMO}{Multiple-Input Multiple-Output}
\newacronym{gp}{GP}{Gaussian Process}
\newacronym{iiot}{IIoT}{Industrial Internet of Things}
\newacronym{csi}{CSI}{Channel State Information}
\newacronym{sgd}{SGD}{Stochastic Gradient Descent}
\newacronym{iid}{i.i.d.}{independent and identically distributed}
\newacronym{ofdm}{OFDM}{Orthogonal Frequency Division Multiplexing}
\newacronym{los}{LOS}{Line-of-Sight}
\newacronym{nlos}{NLOS}{Non-Line-of-Sight}
\newacronym{snr}{SNR}{Signal to Noise Ratio}
\newacronym{rb}{RB}{Resource Block}
\newacronym{6g}{6G}{sixth generation}
\newacronym{ai}{AI}{artificial intelligence}
\newacronym{sfl}{SFL}{Synchronous Federated Learning}
\newacronym{frfl}{FRFL}{Fixed Rate Federated Learning}
\newacronym{pgm}{PGM}{Probabilistic Graphical Model}
\newacronym{hmm}{HMM}{Hidden Markov Model}
\newacronym{elbo}{ELBO}{Evidence Lower Bound}
\newacronym{pmf}{PMF}{Probability Mass Function}
\newacronym{smab}{SMAB}{Stochastic Multi-Armed Bandit}
\newacronym{mab}{MAB}{multi-armed bandit}
\newacronym{mc}{MC}{Monte Carlo}
\newacronym{is}{IS}{Importance Sampling}
\newacronym{dms}{DMS}{discrete memoryless source}
\newacronym{ucb}{UCB}{upper confidence bound}
\newacronym{ser}{SER}{Symbol Error Rate}
\newacronym{sc}{SC}{Semantic Communications}
\newacronym{voi}{VoI}{Value of Information}
\newacronym{nlp}{NLP}{natural language processing}
\newacronym{ts}{TS}{Thompson Sampling}
\newacronym{cmab}{CMAB}{contextual multi-armed bandit}
\newacronym{rccmab}{RC-CMAB}{rate-constrained \gls{cmab}}
\newacronym{rcmab}{R-CMAB}{remote \gls{cmab}}
\newacronym{ib}{IB}{information bottleneck}
\newacronym{merl}{MERL}{maximum entropy reinforcement learning}
\newacronym{fedpm}{$\mathtt{FedPM}$}{Federated Probabilistic Mask Training}
\newacronym{lth}{LTH}{Lottery Ticket Hypothesis}
\newacronym{dp}{DP}{differential privacy}
\newacronym{klm}{$\mathtt{KLMS}$}{KL Minimization with Side Information}

\usepackage{amsthm,latexsym,amssymb,amsmath, amsfonts}
\usepackage{graphicx}
\usepackage{lipsum}
\usepackage[utf8]{inputenc} % allow utf-8 input
\usepackage{hyperref}       % hyperlinks
\usepackage{url}            % simple URL typesetting
\usepackage{booktabs}       % professional-quality tables
\usepackage{amsfonts}       % blackboard math symbols
\usepackage{nicefrac}       % compact symbols for 1/2, etc.
\usepackage{microtype}      % microtypography
\usepackage{xcolor}         % colors
\usepackage{amsthm}
\usepackage{dblfloatfix}  
\usepackage{wrapfig}
\usepackage{array}
\newcolumntype{P}[1]{>{\centering\arraybackslash}p{#1}}
\newcommand{\centered}[1]{\begin{tabular}{c} #1 \end{tabular}}

\usepackage[round]{natbib}
\renewcommand{\bibname}{References}
\renewcommand{\bibsection}{\subsubsection*{\bibname}}

\let\svthefootnote\thefootnote
\newcommand\freefootnote[1]{%
  \let\thefootnote\relax%
  \footnotetext{#1}%
  \let\thefootnote\svthefootnote%
}

\begin{document}

\runningauthor{Isik, Pase, Gunduz, Koyejo, Weissman, Zorzi}

\twocolumn[

%\aistatstitle{Communication-Efficient Federated Learning \\ through Importance Sampling}

\aistatstitle{Adaptive Compression in Federated Learning via Side Information}

\aistatsauthor{  Berivan Isik$^*$ \And Francesco Pase$^*$ \And Deniz Gunduz }

\aistatsaddress{ Stanford University \And  University of Padova \And Imperial College London } 
\aistatsauthor{Sanmi Koyejo \And Tsachy Weissman \And Michele Zorzi}

\aistatsaddress{Stanford University \And Stanford University \And University of Padova}
]

\freefootnote{$^*$Equal contribution.}
\begin{abstract}

The high communication cost of sending model updates from the clients to the server is a significant bottleneck for scalable federated learning (FL). Among existing approaches, state-of-the-art bitrate-accuracy tradeoffs have been achieved using stochastic compression methods -- in which the client $n$ sends a sample from a client-only probability distribution $q_{\phi^{(n)}}$, and the server estimates the mean of the clients' distributions using these samples. However, such methods do not take full advantage of the FL setup where the server, throughout the training process, has \emph{side information} in the form of a global distribution $p_{\theta}$ that is close to the client-only distribution $q_{\phi^{(n)}}$ \emph{in Kullback–Leibler (KL) divergence}. In this work, we exploit this \emph{closeness} between the clients' distributions $q_{\phi^{(n)}}$'s and the side information $p_{\theta}$ at the server, and propose a framework that requires approximately $D_{KL}(q_{\phi^{(n)}}|| p_{\theta})$ bits of communication. We show that our method can be integrated into many existing stochastic compression frameworks to attain the same (and often higher) test accuracy with up to $\mathbf{82}$ 
\textbf{times smaller bitrate} than the prior work -- corresponding to $\mathbf{2,650}$ \textbf{times overall compression}. 
\end{abstract}
\section{Introduction}
\label{introduction}
\looseness=-2
Federated learning (FL), while enabling model training without collecting clients' raw data, suffers from high communication costs due to the model updates communicated from the clients to the server every round \citep{kairouz2021advances}. To mitigate this cost, several communication-efficient FL strategies have been developed that compress the model updates~\citep{lin2018deep, konevcny2016federated, isik2022information, barnes2020rtop}. Many of these strategies adopt a stochastic approach that requires the client $n$ at round $t$ to send a sample $\mathbf{x}^{(t, n)}$ from a client-only distribution $q_{\mathbf{\phi}^{(t, n)}}$ that is only known by the client $n$ upon local training. In turn, the goal of the server is to estimate $\mathbb{E}_{X^{(t,n)} \sim q_{\mathbf{\phi}^{(t,n)}}, \forall n \in [N]} \left [ \frac{1}{N} \sum_{n=1}^N X^{(t,n)} \right ]$ by taking the average of the samples across clients $\frac{1}{N} \sum_{n=1}^N \mathbf{x}^{(t, n)}$. Here, we denote by $N$ the number of clients and by $[N]$ the set $\{1, \dots, N\}$. We show that in many stochastic FL settings, there exists a global distribution $p_{\theta^{(t)}}$ that is known globally by both the server and the clients. This distribution $p_{\theta^{(t)}}$ is close in KL divergence to the client-only distributions $q_{\mathbf{\phi}^{(t,n)}}$'s, which are unknown by the server.\footnote{As we will exemplify later, this global distribution  $p_{\theta^{(t)}}$ is naturally present in many FL frameworks%because it is uniquely extracted by the global model that is being trained
, i.e., we do not introduce or require an extra distribution.} The proposed method, KL Minimization with Side Information (KLMS), exploits this closeness to reduce the cost of communicating samples $\mathbf{x}^{(t,n)}$. We briefly summarize three of such stochastic FL frameworks by pointing to the corresponding global $p_{\theta^{(t)}}$ and client-only $q_{\mathbf{\phi}^{(t,n)}}$ distributions in Section~\ref{sec:intro_examples} as examples of different stochastic FL setups. %: (i) learning probability distributions over subnetworks (or masks), (ii) learning deterministic model parameters using stochastic compressors, and (iii) learning probability distributions over model parameters; and provide a more exhaustive summary in Section~\ref{related} and Appendix~\ref{appendix_prior}. 

%While we detail each step with theoretical guarantees and algorithmic improvements over prior work in Section~\ref{method}; here, we briefly provide the key idea $\mathtt{KLMS}$ relies on: Instead of communicating the deterministic value of a sample $\mathbf{x}^{(t,n)} \sim q_{\phi^{(t,n)}}$, client $n$ can communicate a sample $\mathbf{y}^{(t,n)}$ drawn according to another distribution $\tilde{q}_{\phi^{(t,n)}}$, which is less costly to communicate and is constructed from the pre-data distribution $p_{\theta^{(t)}}$, which is known by the server and the clients, the post-data distribution known by the client, and the importance sampling algorithm in \cite{chatterjee2018sample} as follows:

%compared to $\mathbf{x}^{(t,n)}$, and the discrepancy due to sampling from this distribution is not %significant. 
%To construct $\tilde{q}_{\phi}$, we use the pre-data distribution $p_{\theta^{(t)}}$, which is known by the server and the clients, and the importance sampling algorithm in \cite{chatterjee2018sample} as follows:

\begin{figure*}%{r}{0.34\textwidth}
     \centering
     \begin{center}
     %\vspace{-0.73in}
    \includegraphics[width=\textwidth]{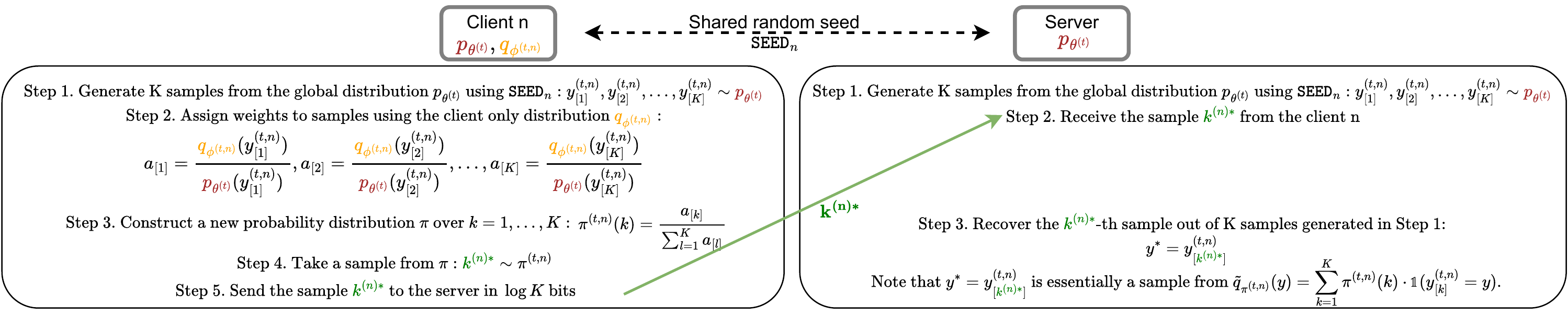}
    \end{center}
     % \vspace{-0.2in}
     \caption{$\mathtt{KLMS}$ Outline. Note that the final sample $y^*$ is a sample from $\tilde{q}_{\pi^{(t, n)}}(\mathbf{y}) = \sum_{k=1}^K \pi^{(t,n)}(k) \cdot \mathbf{1}(\mathbf{y}_{[k]}^{(t, n)} = \mathbf{y})$.% A more detailed description is given in Section~\ref{sec:klm_stochastic_fl} and Appendix~\ref{app_klm}.
     } 
     %\vspace{-0.25in}
     \label{klms_diagram}
\end{figure*}

Before discussing such stochastic FL frameworks $\mathtt{KLMS}$ can be adapted into, we first give a rough outline of how $\mathtt{KLMS}$ actually works in general whenever the global $p_{\theta^{(t)}}$ and client-only $q_{\mathbf{\phi}^{(t,n)}}$ distributions are naturally present in an FL framework. Figure~\ref{klms_diagram} describes the key idea $\mathtt{KLMS}$ relies on (more details in Section~\ref{method} and Appendix~\ref{app_klm}): Instead of communicating the deterministic value of a sample $\mathbf{x}^{(t,n)} \sim q_{\phi^{(t,n)}}$, client $n$ can communicate a sample $\mathbf{y}^{(t,n)}$ from another distribution $\mathbf{y}^{(t,n)} \sim \tilde{q}_{\pi^{(t,n)}}$, which is less costly to communicate compared to $\mathbf{x}^{(t,n)}$, and where the discrepancy due to sampling from this new distribution $\tilde{q}_{\pi^{(t,n)}}$ is not significant. As shown in Figure~\ref{klms_diagram}, to construct $\tilde{q}_{\pi}$, we use the global distribution $p_{\theta^{(t)}}$ (which is known by both the server and the clients) and an importance sampling method as follows: both client $n$ and the server generate $K$ samples from the global distribution $p_{\theta^{(t)}}$ (\textbf{use of side information}); then, the client chooses one of these samples based on the importance weights assigned using the client-only distribution $q_{\phi^{(t,n)}}$ (\textbf{importance sampling}); finally the client sends its choice to the server in $\log K$ bits. We show that this procedure yields an arbitrarily small discrepancy in the estimation when the number of samples in Step 1 in Figure~\ref{klms_diagram} is $K \simeq \exp \left (D_{KL}(q_{\phi^{(n)}} \| p_{\theta}) \right )$, i.e., bitrate is $\log K \simeq D_{KL}(q_{\phi^{(n)}} \| p_{\theta})$ bits, with improvements (specific to the FL setting) over prior work \citep{havasi2018minimal, triastcyn2021dp}.

Clearly, to get the most communication gain out of $\mathtt{KLMS}$, we need global $p_{\theta}$ and client-only $q_{\phi^{(n)}}$ distributions that are close in KL divergence. We show the existence of such distributions in many stochastic FL frameworks with concrete examples in Section~\ref{sec:intro_examples}. Each of these FL frameworks we will cover, namely \texttt{FedPM}~\citep{isik2023sparse}, \texttt{QSGD}~\citep{alistarh2017qsgd}, and Federated \texttt{SGLD}~\citep{vono2022qlsd}; naturally induces a client-only distribution $q_{\mathbf{\theta}^{(t,n)}}$ that clients want to send a sample from, and a global distribution $p_{\theta^{(t)}}$ that is available to both the clients and the server -- playing the role of side information. \textbf{Note that these distributions are already present in the original frameworks without any additional assumption or modification from us.} In each case, these distributions are expected to become closer in KL divergence as training progresses as we will explain in Section~\ref{sec:intro_examples}. We show that $\mathtt{KLMS}$ reduces the communication cost down to this \emph{fundamental quantity} (KL divergence) in each scenario, resulting in up to $\mathbf{82}$ \textbf{times improvement} in communication efficiency over \texttt{FedPM}, \texttt{QLSD}, and \texttt{QSGD} among other non-stochastic competitive baselines. To achieve this efficiency, we use an importance sampling algorithm  \citep{chatterjee2018sample} -- thus extending the previous theoretical guarantees to the distributed setting. Different from prior work that used importance sampling in the centralized setting to compress model parameters \citep{havasi2018minimal} or focused on differential privacy implications \citep{shah2022optimal, triastcyn2021dp}, $\mathtt{KLMS}$ captures the side information that is already present in many FL frameworks by selecting more natural global $p_{\mathbf{\theta}^{(t)}}$ and client-only $q_{\mathbf{\phi}^{(t,n)}}$ distributions, and optimizes the bit allocation across both the training rounds and the model coordinates in an adaptive way to achieve the optimal bitrate, while also eliminating a hyperparameter required by prior work \citep{havasi2018minimal, triastcyn2021dp}. We note that arbitrary choices of global and client-only distributions as in these works lead to suboptimally as the KL divergence between two arbitrary distributions would not be necessarily small. However, we discover natural choices for both distributions in existing FL frameworks \textbf{(without any modification on the frameworks)} that yield significantly smaller KL divergence and hence superior compression than prior work. Our contributions:
\begin{enumerate}
    \item We propose a road map to utilize various forms of side information available to both the server and the clients to reduce the communication cost in FL. We give concrete examples of how to send model updates under different setups. See Section~\ref{sec:intro_examples} for details.
\item We extend the importance sampling results in~\citep{chatterjee2018sample} to the distributed setting. 

\item We propose an adaptive bit allocation strategy that eliminates a hyperparameter required by prior work, and allows a better use of the communication budget across the model coordinates and rounds.

\item We demonstrate the efficacy of \texttt{KLMS} on MNIST, EMNIST, CIFAR-10, and CIFAR-100 datasets with up to $\mathbf{82}$ \textbf{times gains} in bitrate over relevant baselines.  This corresponds to up to an \textbf{overall 2,650 times compression} without a significant accuracy drop.
\end{enumerate}

\section{Related Work}
\label{related}
\textbf{Communication-Efficient FL}: Existing frameworks reduce the communication cost by sparsification~\citep{aji2017sparse, wang2018atomo, lin2018deep}, quantization~\citep{suresh2017distributed, vono2022qlsd, wen2017terngrad, mayekar2021wyner}, low-rank factorization~\citep{basat2022quick, mohtashami2022masked, vogels2019powersgd}, sketching~\citep{rothchild2020fetchsgd, song2023sketching}; or by training sparse subnetworks instead of the full model~\citep{isik2023sparse, li2020lotteryfl, li2021fedmask, liu2021fedprune}. Among them, those based on stochastic updates have shown success over the deterministic ones in similar settings. For instance, as will become clear in Section~\ref{sec:intro_examples}, for finding sparse subnetworks within a large random model, \texttt{FedPM}~\citep{isik2023sparse} takes a stochastic approach by training a probability mask and outperforms other methods that find sparse subnetworks deterministically~\citep{li2021fedmask, mozaffari2021frl, vallapuram2022hidenseek} with significant accuracy and bitrate gains. Similarly, for the standard FL setting (training model parameters), \texttt{QSGD}~\citep{alistarh2017qsgd} is an effective stochastic quantization method -- outperforming most other quantization schemes such as \texttt{SignSGD}~\citep{bernstein2018signsgd} and \texttt{TernGrad}~\citep{wen2017terngrad} by large margins. Lastly, in the Bayesian FL setting, \texttt{QLSD}~\citep{vono2022qlsd} proposes a Bayesian counterpart of \texttt{QSGD}, and performs better than other baselines \citep{chen2021fedbe, plassier2021dg}. While all these stochastic approaches already perform better than the relevant baselines, in this work, we show that they still do not take full advantage of the \emph{side information} (or global distribution) available to the server. We provide a guideline on how to find useful side information under each setting and introduce $\mathtt{KLMS}$ that reduces the communication cost (by $82$ times over the baselines) to the fundamental distance between the client's distribution that they want to communicate samples from and the side information at the server.

\textbf{Importance Sampling:} Our strategy is inspired by the importance sampling algorithm studied in \citep{chatterjee2018sample, harsha2007communication, theis2022algorithms, li2018strong, flamich2024faster}, and later applied for model compression~\citep{havasi2018minimal}, learned image compression~\citep{flamich2020compressing, flamich2022fast}, and compressing differentially private mechanisms~\citep{shah2022optimal, triastcyn2021dp, isik2023exact}. One relevant work to ours is \citep{havasi2018minimal}, which applies the importance sampling strategy to compress Bayesian neural networks. Since the model size is too large to be compressed at once, they compress fixed-size blocks of the model parameters separately and independently. As we elaborate in Section~\ref{method}, this can be done much more efficiently by choosing the block size adaptively based on the information content of each parameter. Another relevant work is \texttt{DP-REC}~\citep{triastcyn2021dp}, which again applies the importance sampling technique to compress the model updates in FL, while also showing differential privacy implications. However, since their training strategy is fully deterministic, the choice of global and client-only distributions is somewhat arbitrary. Instead, in our work, the goal is to exploit the available side information to the full extent by choosing natural global and client-only distributions -- which improves the communication efficiency over \texttt{DP-REC} significantly. Another factor in this improvement is the adaptive bit allocation strategy mentioned above. Our experimental results demonstrate that these two improvements are indeed critical for boosting the accuracy-bitrate tradeoff. Finally, we extend the theoretical guarantees of importance sampling, which quantifies the required bitrate for a target discrepancy (due to compression), to the distributed setting, where we can recover the existing results in \citep{chatterjee2018sample} as a special case by setting $N=1$.

\section{Preliminaries}
\label{sec:intro_examples}

We now briefly summarize three examples of stochastic FL frameworks that $\mathtt{KLMS}$ can be integrated into by highlighting the natural choices for global $p_{\theta}$ and client-only $q_{\phi^{(n)}}$ distributions.

\textbf{FedPM}~\citep{isik2023sparse} freezes the parameters of a randomly initialized network and finds a subnetwork inside it that performs well with the initial random parameters. To find the subnetwork, the clients receive a global probability mask $\theta^{(t)}\in [0, 1]^d$ from the server that determines, for each parameter, the probability of retaining it in the subnetwork; set this as their local probability mask $\phi^{(t, n)} \gets \theta^{(t)}$; and train only this mask (not the frozen random parameters) during local training. At inference, a sample $x^{(t,n)} \in \{0, 1\}^d$ from the Bernoulli distribution $\text{Bern}(\cdot;\phi^{(t,n)})$ is taken, and multiplied element-wise with the frozen parameters of the network, obtaining a pruned random subnetwork, which is then used to compute the model outputs. Communication consists of three stages: (i) clients update their local probability masks $\phi^{(t,n)}$ through local training; (ii) at the end of local training, they send a sample $x^{(t, n)} \sim \text{Bern}(\cdot; \phi^{(t,n)})$ to the server;
(iii) the server aggregates the samples $\frac{1}{N} \sum_{n=1}^N x^{(t,n)}$, updates the global probability mask $\theta^{(t+1)}$, and broadcasts the new mask to the clients for the next round. \texttt{FedPM} achieves state-of-the-art results in accuracy-bitrate tradeoff with around $1$ bit per parameter (bpp). (full description in Appendix~\ref{appendix:fedpm}). 
As the model converges, the global probability mask $\theta^{(t)}$ and clients' local probability masks $\mathbf{\phi}^{(t,n)}$ get closer to each other (see Figures~\ref{kl_layer_round} and~\ref{fig:ablation_adaptve} for the trend of $D_{KL}(q_{\mathbf{\phi}^{(t,n)}}||p_{\theta^{(t)}}$) over time). However, no matter how close they are, \texttt{FedPM} employs approximately the same bitrate for communicating a sample from $\text{Bern}(\cdot;\phi^{(t,n)})$ to the server that knows $p_{\theta^{(t)}}$. We show that this strategy is suboptimal and applying $\mathtt{KLMS}$ with the global probability distribution $\text{Bern}(\cdot;\theta^{(t)})$ as the global distribution $p_{\theta^{(t)}}$, and the local probability distribution $\text{Bern}(\cdot;\phi^{(t,n)})$ as the client-only distribution $q_{\phi^{(t,n)}}$, provides up to $82$ times gain in compression over \texttt{FedPM}.

\textbf{QSGD}~\citep{alistarh2017qsgd}, unlike the stochastic approach in \texttt{FedPM} to train a probabilistic mask, is proposed to train a deterministic set of parameters. However, \texttt{QSGD} is itself a stochastic quantization operation. More concretely, \texttt{QSGD} quantizes each coordinate $\mathbf{v}^{(t,n)}_i$ using the following \texttt{QSGD} distribution $p_{\texttt{QSGD}}(\cdot)$):

\resizebox{\columnwidth}{!}{
  \begin{minipage}{\columnwidth}
    \begin{align}
    \begin{split}
    & p_{\texttt{QSGD}}\left(\mathbf{\hat{v}}^{(t,n)}_i\right) = \\
    &\begin{cases} \frac{s |\mathbf{v}^{(t,n)}_i|}{\|\mathbf{v}^{(t,n)}\|} -  \left\lfloor \frac{s |\mathbf{v}^{(t,n)}_i|}{\|\mathbf{v}^{(t,n)}\|} \right\rfloor  \ & \text{if} \   \mathbf{\hat{v}}^{(t,n)}_i = A(\mathbf{v}^{(t,n)}_i) \\
     \ 1 - \frac{s |\mathbf{v}^{(t,n)}_i|}{\|\mathbf{v}^{(t,n)}\|} +  \left\lfloor \frac{s |\mathbf{v}^{(t,n)}_i|}{\|\mathbf{v}^{(t,n)}\|} \right\rfloor \ & \text{if} \ \mathbf{\hat{v}}^{(t,n)}_i = B(\mathbf{v}^{(t,n)}_i)
   \end{cases},
\end{split}
   \label{eq:qsgd}
\end{align}
\end{minipage}
}
where 

\resizebox{0.8\columnwidth}{!}{
  \begin{minipage}{0.8\columnwidth}
\begin{align*}
A(\mathbf{v}^{(t,n)}_i) &= \frac{\|\mathbf{v}^{(t,n)}\| \cdot \text{sign}(\mathbf{v}^{(t,n)}_i)}{s} \left ( \left\lfloor \frac{s |\mathbf{v}^{(t,n)}_i|}{\|\mathbf{v}^{(t,n)}\|} \right\rfloor + 1\right ), \\
B(\mathbf{v}^{(t,n)}_i) &=\frac{\|\mathbf{v}^{(t,n)}\| \cdot \text{sign}(\mathbf{v}^{(t,n)}_i)}{s} \left\lfloor \frac{s |\mathbf{v}^{(t,n)}_i|}{\|\mathbf{v}^{(t,n)}\|} \right\rfloor,
\end{align*}
\end{minipage}}

and $s$ is the number of quantization levels (full description in Appendix~\ref{appendix:qsgd}). \texttt{QSGD} takes advantage of the empirical distribution of the quantized values (large quantized values are less frequent) by using Elias coding to encode them -- which is the preferred code when the small values to encode are much more frequent than the larger values \citep{elias1975universal}. However, \texttt{QSGD} still does not fully capture the distribution of the quantized values since Elias coding is not adaptive to the data. We fix this mismatch by applying $\mathtt{KLMS}$ with the \texttt{QSGD} distribution $p_{\texttt{QSGD}}(\cdot)$ as the client-only distribution $q_{\mathbf{\phi}^{(t,n)}}$, and the empirical distribution induced by the historical updates at the server from the previous round as the global distribution $p_{\theta^{(t)}}$. These two distributions are expected to be \emph{close} to each other due to the temporal correlation across rounds, as previously reported by \cite{jhunjhunwala2021leveraging, ozfatura2021time}. We demonstrate that $\mathtt{KLMS}$ exploits this closeness and outperforms vanilla \texttt{QSGD} with $10$ times lower bitrate.

\textbf{Federated SGLD}~\citep{lsd2021} is a Bayesian FL framework that learns a global posterior distribution $p_{\mathbf{\theta}}$ over the model parameters using clients' local posteriors $q_{\phi^{(n)}}$. A state-of-the-art method~\citep{vono2022qlsd} is the FL counterpart of Stochastic Gradient Langevin Dynamics (\texttt{SGLD}) \citep{welling2011bayesian}, which uses a Markov Chain Monte Carlo algorithm. Concretely, the global posterior distribution is assumed to be proportional to the product $ p_{\theta^{(t)}} \sim \prod_{n=1}^N e^{-U(\phi^{(t, n)})}$ of $N$ local unnormalized posteriors associated with each client, expressed as potential functions $\{U(\phi^{(t, n)})\}_{n=1}^N$. At each round, the clients' local posteriors are initialized with the global posterior $\phi^{(n,t)} \gets \theta^{(t)}$. Then, the clients compute an unbiased estimate of their gradients $H(\phi^{(t, n)}) = \frac{|D^{(n)}|}{|\mathcal{S}^{(t,n)}|} \sum_{j \in \mathcal{S}^{(t, n)}} \nabla U_j(\phi^{(t, n)})$, where $|D^{(n)}|$ is the size of the local dataset of client $n$, and $\mathcal{S}^{(t,n)}$ is the batch of data used to estimate the gradient. They then communicate these estimates to the server to compute

\begin{align}
    \label{eq:sgld_agg}
    \theta^{(t+1)} = \theta^{(t)} - \gamma \sum_{n=1}^N H(\phi^{(t, n)}) + \sqrt{2\gamma} \xi^{(t)}, 
\end{align}

where $\xi^{(t)}$ is a sequence of i.i.d. standard Gaussian random variables. (See Appendix~\ref{appendix:qlsd} for the details.) As reported in \citep{lsd2021, vono2022qlsd}, the sequence of global updates $\theta^{(t)}$ converges to the posterior sampling. Notice that the clients communicate their gradient vectors $H(\phi^{(t, n)})$ to the server at every round, which is as large as the model itself. To reduce this communication cost, \cite{vono2022qlsd} propose a compression algorithm called \texttt{QLSD} that stochastically quantizes the updates with essentially the Bayesian counterpart of \texttt{QSGD} \citep{alistarh2017qsgd}. However, neither \texttt{QLSD} nor the other compression baselines in the Bayesian FL literature \citep{chen2021fedbe, el2020distributed, plassier2021dg} take full advantage of the stochastic formulation of the Bayesian framework, where the server and the clients share side information (the global posterior $p_{\theta^{(t)}}$) that could be used to improve the compression gains. Instead, they quantize the updates ignoring this side information. This approach is suboptimal since (i) the precision is already degraded in the quantization step, and (ii) the compression step does not account for the side information $p_{\theta^{(t)}}$. We show that we can exploit this inherent stochastic formulation of Bayesian FL by applying $\mathtt{KLMS}$ with the global posterior distribution as the global distribution $p_{\mathbf{\theta}^{(t)}}$, and the local posterior distribution as the client-only distribution $q_{\mathbf{\phi}^{(t,n)}}$. In addition to benefiting from the side information, $\mathtt{KLMS}$ does not restrict the message domain to be discrete (as opposed to the baselines) and can reduce the communication cost by $5$ times, while also achieving higher accuracy than the baselines.
\looseness=-1

\section{KL Divergence Minimization with Side Information (\texttt{KLMS})}
\label{method}
We first describe our approach, $\mathtt{KLMS}$, in Section~\ref{sec:klm_stochastic_fl} together with theoretical guarantees; then, in Section~\ref{sec:klm_adaptive}, we introduce our adaptive bit allocation strategy to optimize the bitrate across training rounds and model coordinates to reduce the compression rate; finally, in Section~\ref{sec:klm_examples}, we give four concrete examples where $\mathtt{KLMS}$ significantly boosts the accuracy-bitrate tradeoff.

\subsection{\texttt{KLMS} for Stochastic FL Frameworks} \label{sec:klm_stochastic_fl}

We propose \texttt{KLMS} as a general recipe to be integrated into many existing (stochastic) FL frameworks to improve their accuracy-bitrate performance \emph{significantly}. The main principle in $\mathtt{KLMS}$ is grounded in three ideas:

\begin{enumerate}
    \item In many existing FL frameworks, the updates from clients to the server are samples drawn from some optimized client-only distributions, e.g., \texttt{QSGD} and \texttt{FedPM}. 

    \item Sending a \emph{random} sample from a distribution can be done much more efficiently than first taking a sample from the same distribution, and then sending its \emph{deterministic} value~\citep{theis2022algorithms}.

    \item The knowledge acquired from the historical updates, available both at the server and the clients, can help reduce the communication cost drastically by playing the role of \emph{temporal} side information. 
\end{enumerate}

$\mathtt{KLMS}$ is designed to reduce the communication cost in FL by taking advantage of the above observations. It relies on shared randomness between the clients and the server in the form of a shared random \texttt{SEED} (i.e., they can generate the same pseudo-random samples from a given distribution) and on the side information available to the server and the clients. Without restricting ourselves to any specific FL framework (we will do this in Section~\ref{sec:klm_examples}), suppose the server and the clients share a global distribution $p_{\theta^{(t)}}$ and each client has a client-only distribution $q_{\phi^{(t,n)}}$ after local training. As stated in Section~\ref{introduction}, the  server aims to compute $\mathbb{E}_{X^{(t,n)} \sim q_{\mathbf{\phi}^{(t,n)}}, \forall n \in [N]} \left [ \frac{1}{N} \sum_{n=1}^N X^{(t, n)} \right ]$ after each round. While this can be done by simply communicating samples $\mathbf{x}^{(t,n)} \sim q_{\mathbf{\phi}^{(t,n)}}$, it is actually sufficient for the server to obtain any other set of samples from the same distribution $q_{\mathbf{\phi}^{(t,n)}}$ (or another distribution that is close to $q_{\mathbf{\phi}^{(t,n)}}$ in KL divergence). Therefore, instead of a specific realization $\mathbf{x}^{(t, n)} \sim q_{\mathbf{\phi}^{(t, n)}}$, $\mathtt{KLMS}$ sends a sample $\mathbf{y}^{(t, n)}$ from some other distribution $\tilde{q}_{\mathbf{\pi}^{(t, n)}}$ such that (i) it is less costly to communicate a sample from $\tilde{q}_{\mathbf{\pi}^{(t, n)}}$ than $q_{\mathbf{\phi}^{(t, n)}}$ and (ii) the discrepancy

\begin{align}
\begin{split}
    E = \Bigg | & \mathbb{E}_{Y^{(t,n)} \sim \tilde{q}_{\mathbf{\pi}^{(t,n)}}, \forall n \in [N]} \left [\frac{1}{N} \sum_{n=1}^N Y^{(t, n)}\right ] - \\ & \mathbb{E}_{X^{(t,n)} \sim q_{\mathbf{\phi}^{(t,n)}}, \forall n \in [N]} \left [ \frac{1}{N} \sum_{n=1}^N X^{(t, n)} \right ] \Bigg | 
\end{split}
    \label{eq:discrepancy}
\end{align}

is sufficiently small. Motivated by this, $\mathtt{KLMS}$ runs as follows (by referring to the steps in Figure~\ref{klms_diagram}): 

 \textbf{Step 1 at the server \& client (side information):} The server and the client generate the same $K$ samples from the global distribution with a shared random seed.

\textbf{Steps 2-4 at the client (importance sampling):} The client assigns importance weights to each of the $K$ samples to construct a new distribution over them. It then chooses one of the $K$ samples from this new distribution to send its index in $\log K$ bits to the server. 

\textbf{Steps 2-3 at the server (importance sampling):} The server picks the sample with the received index from the $K$ samples it generated in Step 1.

In Theorem~\ref{thm2_main}, we show that the discrepancy in (\ref{eq:discrepancy}) is small when $K \simeq \exp{ \left ( D_{KL}(q_{\phi} || p_{\theta}) \right )}$. We actually prove it for a general measurable function $f(\cdot)$ over $Y^{(t, n)}$'s, for which the discrepancy in (\ref{eq:discrepancy}) is a special case when $f(\cdot)$ is the identity. We note that previous results on the single-user scenario ($N=1$) \citep{chatterjee2018sample, havasi2018minimal} are special cases of our more general framework with $N$ users.

\begin{theorem} \label{thm2_main} Let $p_{\theta}$ and $q_{\phi^{(n)}}$ for $n=1, \dots, N$ be probability distributions over set $\mathcal{X}$ equipped with some sigma-algebra. Let $X^{(n)}$ be
an $\mathcal{X}$-valued random variable with law $q_{\phi^{(n)}}$. Let $r \geq 0$ and $\tilde{q}_{\pi^{(n)}}$ for $n=1, \dots, N$ be discrete distributions each constructed by $K^{(n)}= \exp{ \left (D_{KL}(q_{\phi^{(n)}} || p_{\theta} )  + r  \right )}$ samples $\{\mathbf{y}^{(n)}_{[k]}\}_{k=1}^{K^{(n)}}$ from $p_{\theta}$ defining $\pi^{(n)}(k) = \frac{q_{\phi^{(n)}}(\mathbf{y}^{(n)}_{[k]}) / p_{\theta}(\mathbf{y}^{(n)}_{[k]})}{\sum_{l=1}^{K^{(n)}} q_{\phi^{(n)}}(\mathbf{y}^{(n)}_{[l]}) / p_{\theta}(\mathbf{y}^{(n)}_{[l]})} $. Furthermore, for measurable function $f(\cdot)$, let $\|f\|_{\mathbf{q_{\phi}}} = \sqrt{\mathbb{E}_{X^{(n)} \sim q_{\phi^{(n)}}, \forall n \in [N]} [ (\frac{1}{N}\sum_{n=1}^N f(X^{(n)}))^2] }$ be its 2-norm under $\mathbf{q_{\phi}} = q_{\phi^{(1)}}, \dots, q_{\phi^{(N)}}$ and let

\resizebox{0.06\linewidth}{!}{
 \begin{minipage}{0.06\linewidth}
\begin{align}
    \epsilon = &\left( e^{-Nr/4} + 2 \sqrt{  \prod_{n=1}^N \mathbb{P} ( \log ( q_{\phi^{(n)}}/p_{\theta} ) > D_{KL}(q_{\phi^{(n)}} \| p_{\theta} ) + r/2) }\right )^{1/2}.
\end{align}
\end{minipage}}

Defining $\tilde{q}_{\pi^{(n)}}$ over $\{\mathbf{y}_{[k]}^{(n)}\}_{k=1}^{K^{(n)}}$ as $\tilde{q}_{\pi^{(n)}}(\mathbf{y}) = \sum_{k=1}^{K^{(n)}} \pi^{(n)}(k) \cdot \mathbf{1}(\mathbf{y}_{[k]}^{(n)} = \mathbf{y})$, it holds that

\resizebox{0.9\linewidth}{!}{
 \begin{minipage}{0.9\linewidth}
\begin{align*}
    \mathbb{P} & \Bigg ( \Bigg | \displaystyle \mathop{\mathbb{E}}_{Y^{(n)} \sim \tilde{q}_{\mathbf{\pi}^{(n)}}, \forall n } \left [\frac{1}{N} \sum_{n=1}^Nf(Y^{(n)}) \right ] - \\
   &  \displaystyle \mathop{\mathbb{E}}_{X^{(n)} \sim q_{\mathbf{\phi}^{(n)}}, \forall n} \left [ \frac{1}{N} \sum_{n=1}^Nf(X^{(n)}) \right ]  \Bigg | \geq \frac{2 \|f\|_{\mathbf{q_{\phi}}} \epsilon}{1 - \epsilon} \Bigg ) \leq 2 \epsilon,
\end{align*}
\end{minipage}
}

where $\tilde{q}_{\pi^{(n)}}$ is defined over $\{\mathbf{y}_{[k]}^{(n)}\}_{k=1}^{K^{(n)}}$ as $\tilde{q}_{\pi^{(n)}}(\mathbf{y}) = \sum_{k=1}^{K^{(n)}} \pi^{(n)}(k) \cdot \mathbf{1}(\mathbf{y}_{[k]}^{(n)} = \mathbf{y})$.
\end{theorem}
See Appendix~\ref{appendix_proofs} for the proof. This result implies that when $K^{(n)} \simeq \exp \left( D_{KL}(q_{\phi^{(t,n)}} \| p_{\theta^{(t)}} ) \right )$, the discrepancy in (\ref{eq:discrepancy}) is small. In practice, as we explain in Section~\ref{sec:klm_adaptive}, we work on blocks of parameters such that $D_{KL}(q_{\phi^{(t,n)}} \| p_{\theta^{(t)}} )$ for each block is the same for all clients $n \in [N]$. Hence, we omit the superscript $(n)$ from $K^{(n)}$  and denote the number of samples by $K$ for each client. In Appendix~\ref{sec:sampling_experiments}, we experiment on a toy model and observe that, for a fixed $K$, the discrepancy in (\ref{eq:discrepancy}) gets smaller as the number of clients $N$ increases, gaining from the client participation in each round.

\subsection{Adaptive Block Selection for Optimal Bit Allocation} \label{sec:klm_adaptive}

Prior works that have applied importance sampling for Bayesian neural network compression \citep{havasi2018minimal}, or for differentially private communication in FL \citep{triastcyn2021dp} split the model into several fixed-size blocks of parameters, and compress each block separately and independently to avoid the high computational cost -- which exponentially increases with the number of parameters $d$. After splitting the model into fixed-size blocks with $S$ parameters each, \cite{havasi2018minimal, triastcyn2021dp} choose a single fixed $K$ (number of samples generated from $p_{\theta^{(t)}}$) for each block no matter what the KL divergence is for different blocks. This yields the same bitrate $\frac{\log K}{S}$ for every model parameter. Furthermore, \cite{triastcyn2021dp} use the same $K$ throughout training without considering the variation in KL divergence over rounds. However, as illustrated in Figure~\ref{kl_layer_round}, KL divergence varies significantly across different model layers and across rounds. Hence, spending the same bitrate $\frac{\log K}{S}$ for every parameter at every round is highly suboptimal since it breaks the condition in Theorem~\ref{thm2_main}. 

To fix this, we propose an adaptive block selection mechanism, where the block size is adjusted such that the KL divergence for each block is the same and equal to a target value, $D_{KL}^{\text{target}}$. This way, the optimal $K$ for each block is the same and approximately equal to $D_{KL}^{\text{target}}$, and we do not need to set the block size $S$ ourselves, which was a  hyperparameter to tune in \citep{havasi2018minimal, triastcyn2021dp}. Different from the fixed-size block selection approach in \citep{havasi2018minimal, triastcyn2021dp}, the adaptive approach requires describing the locations of the adaptive-size blocks, which adds overhead to the communication cost. However, exploiting the temporal correlation across rounds can make this overhead negligible. More specifically, we first let each client find their adaptive-size blocks, each having KL divergence equal to $D_{KL}^{\text{target}}$, in the first round. Then the clients communicate the locations of these blocks to the server, which are then aggregated to find the new global indices to be broadcast to the clients. At later rounds, the server checks if, on average, the new KL divergence of the previous blocks is still sufficiently close to the target value $D_{KL}^{\text{target}}$. If so, the same adaptive-size blocks are used in that round. Otherwise, the client constructs new blocks, each having KL divergence equal to $D_{KL}^{\text{target}}$, and updates the server about the new locations. Our experiments indicate that this update occurs only a few times during the whole training. Therefore, it adds only a negligible overhead on the average communication cost across rounds. We provide the pseudocodes for $\mathtt{KLMS}$ with both fixed- and adaptive-size blocks in Appendix~\ref{app_klm}.

\subsection{Examples of $\mathtt{KLMS}$ Adaptated to Well-Known Stochastic FL Frameworks} \label{sec:klm_examples}

In this section, we provide four concrete examples illustrating how $\mathtt{KLMS}$ can be naturally integrated into different FL frameworks with natural choices of global and client-only distributions. Later, in Section~\ref{experiments}, we present experimental results showing the empirical improvements $\mathtt{KLMS}$ brings in all these cases. The corresponding pseudocodes are given in Appendix~\ref{app_details_klm_examples}.

\textbf{FedPM-KLMS:} As described in Section~\ref{sec:intro_examples}, in \texttt{FedPM}, the server holds a global probability mask, which parameterizes a probability distribution over the mask parameters -- indicating for each model parameter its probability of remaining in the subnetwork. Similarly, each client obtains a local probability mask after local training -- parameterizing their locally updated probability assignment for each model parameter to remain in the subnetwork. Parameterizing the global distribution $p_{\theta^{(t)}}$ by the global probability mask $\theta^{(t)}$ and the client-only distribution $q_{\phi^{(t,n)}}$ by the local probability mask $\phi^{(t,n)}$ is only natural since the goal in \texttt{FedPM} is to send a sample from the local probability distribution $\text{Bern}(\cdot; \phi^{(t,n)})$ with as few bits as possible. This new framework, \texttt{FedPM-KLMS}, provides $\mathbf{82}$ \textbf{times reduction in bitrate over vanilla \texttt{FedPM}} -- corresponding to an \textbf{overall} $\mathbf{2,650}$ \textbf{times compression rate}.

\begin{figure}
     \centering
     \begin{center}
    \includegraphics[width=0.4\textwidth]{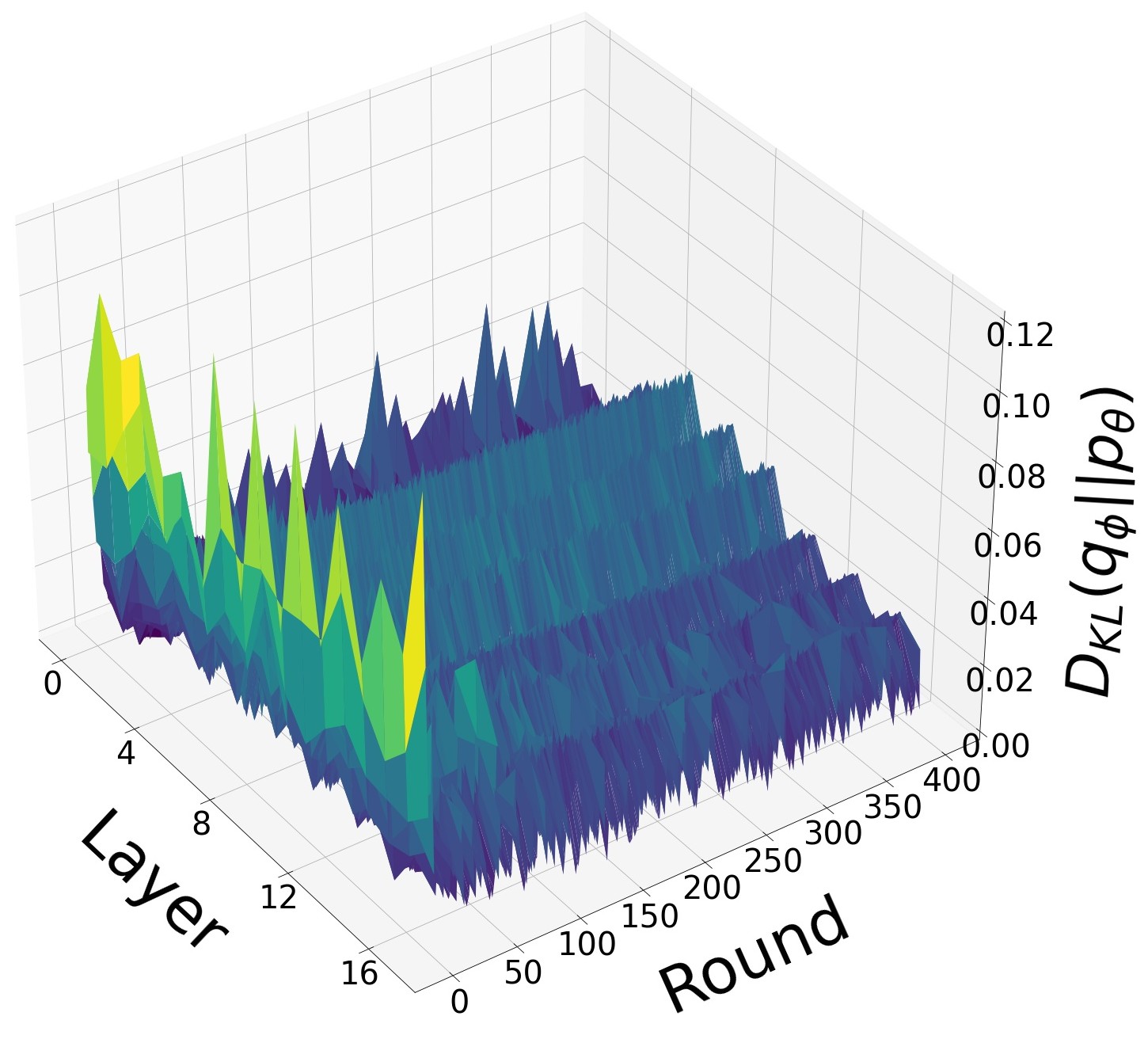}
    \end{center}
     \caption{Average KL divergence between the client-only and global distributions, for different layers and rounds (\texttt{FedPM} used to train \texttt{CONV6} on CIFAR-$10$). 
     } 
     \vspace{-0.2in}
     \label{kl_layer_round}
\end{figure}

\begin{table*}[]
\centering
\caption{\texttt{FedPM-KLMS}, \texttt{QSGD-KLMS}, and \texttt{SignSGD-KLMS} against \texttt{FedPM}, \texttt{QSGD}, \texttt{SignSGD}, \texttt{TernGrad}, \texttt{DRIVE}, \texttt{EDEN}, \texttt{MARINA}, \texttt{FedMask}, and \texttt{DP-REC} with i.i.d. split and full client participation. Note that the bitrate is $32$ bpp without any compression. Overall, \texttt{FedPM-KLMS} achieves the highest accuracy with the lowest bitrate around $0.014$ for each dataset -- corresponding to an overall $2,300$ times compression. Other \texttt{KLMS} integrations (\texttt{QSGD-KLMS}, \texttt{SignSGD-KLMS}) similarly reduce the bitrate up to $66$ times over the vanilla frameworks (\texttt{QSGD}, \texttt{SignSGD}). 
    }
\resizebox{\linewidth}{!}{
\begin{tabular}{lcccccccc}
             & \multicolumn{2}{c}{CIFAR-10 (\texttt{CONV6})} & \multicolumn{2}{c}{CIFAR-100 (ResNet-18)} & \multicolumn{2}{c}{MNIST (\texttt{CONV4})} & \multicolumn{2}{c}{EMNIST (\texttt{CONV4})} \\
             \toprule
      & \multicolumn{1}{c}{Acc.} & \multicolumn{1}{c}{Bitrate (bpp)}               & \multicolumn{1}{c}{Acc.}  & \multicolumn{1}{c}{Bitrate (bpp)}                 & \multicolumn{1}{c}{Acc.}  & \multicolumn{1}{c}{Bitrate (bpp)}                & \multicolumn{1}{c}{Acc.}  & \multicolumn{1}{c}{Bitrate (bpp)}      \\
      \toprule
\texttt{TernGrad}      & 0.680 &  1.10             &    0.220          &   1.07            & 0.980              &     1.05         &  0.870          & 1.1    \\
\texttt{DRIVE}      & 0.760 & 0.89              &    0.320          &  0.54             &      0.994         &   0.91            &    0.883         &   0.90  \\
\texttt{EDEN}     & 0.760 & 0.89               &    0.320          &    0.54           &      0.994         &  0.91            &  0.883          &  0.90   \\
\texttt{MARINA} & 0.690 & 2.12               &    0.260          &    2.18           &      0.991         &  2.01            &  0.867          &  2.04   \\
\texttt{FedMask}     &  0.620 & 1.00              &    0.180          &    1.00           & 0.991              &    1.00          &     0.862       & 1.00    \\
\texttt{DP-REC}     & 0.720 & 1.12               &    0.280          &  1.06             &     0.991          &       1.00       &   0.885         &  1.10   \\
\midrule
\texttt{SignSGD}     & 0.705 & 0.993               &    0.230           &  0.999             &    0.990           &   0.999           &   0.873         &  1.000   \\
\texttt{SignSGD-KLMS}(ours) &   \textbf{0.745}            & \textbf{0.040  ($\mathbf{\times 25}$ lower)}       & \textbf{0.259}            &     \textbf{0.042   ($\mathbf{\times 25}$ lower)}      &   \textbf{0.9930}       & \textbf{0.041  ($\mathbf{\times 24}$ lower)}  & \textbf{0.880}   &  \textbf{0.044 ($\mathbf{\times 23}$ lower)}   \\
\texttt{SignSGD-KLMS}(ours) &    \textbf{0.739}           &\textbf{0.015 ($\mathbf{\times 66}$ lower)}          & \textbf{0.250}           &  \textbf{0.018  ($\mathbf{\times 56}$ lower)}      &  \textbf{0.9918}        & \textbf{0.023  ($\mathbf{\times 43}$ lower)}  &  \textbf{0.875} &   \textbf{0.025 ($\mathbf{\times 40}$ lower)} \\
\midrule
\texttt{QSGD} &    0.753             &     0.072         &  0.335             &   0.150            &   0.994            &  0.130    & 0.884   & 0.150     \\
\texttt{QSGD-KLMS}(ours)           &   \textbf{  0.761  }           & \textbf{0.035  ($\mathbf{\times 2}$ lower)}       & \textbf{0.327}           &    \textbf{0.074  ($\mathbf{\times 2.2}$ lower)}        &     \textbf{0.9940}         & \textbf{0.041  ($\mathbf{\times 3.2}$ lower)} & \textbf{0.884}   &  \textbf{0.042 ($\mathbf{\times 3.6}$ lower)} \\
\texttt{QSGD-KLMS}(ours) &        \textbf{0.755 }    &      \textbf{0.014  ($\mathbf{\times 5}$ lower)}  & \textbf{0.320}        &   \textbf{0.020    ($\mathbf{\times 7.5}$ lower)}     & \textbf{0.9935}        & \textbf{0.019 ($\mathbf{\times 6.8}$ lower)}   &  \textbf{0.883} &  \textbf{0.022   ($\mathbf{\times 6.8}$ lower)} \\
\midrule
\texttt{FedPM}     & 0.787  &  0.845             &    0.470          &   0.88            &  0.995             &  0.99            &   0.890         &  0.890   \\
\texttt{FedPM-KLMS}(ours) &    \textbf{0.787  }           &    \textbf{0.070   ($\mathbf{\times 12}$ lower)}      &  \textbf{0.469}          &  \textbf{0.072    ($\mathbf{\times 13}$ lower)}      &     \textbf{0.9945}    &  \textbf{0.041 ($\mathbf{\times 24}$ lower)}   & \textbf{0.888}   & \textbf{0.034 ($\mathbf{\times 26}$ lower)}    \\
\texttt{FedPM-KLMS}(ours)           &   \textbf{0.786}             &   \textbf{0.014 ($\mathbf{\times 60}$ lower)}       & \textbf{0.455}             &   \textbf{0.018    ($\mathbf{\times 49}$ lower)}      &        \textbf{0.9943}    &  \textbf{0.014   ($\mathbf{\times 71}$ lower)}  & \textbf{0.885}    &  \textbf{0.017 ($\mathbf{\times 52}$ lower)}\\   
\bottomrule
\end{tabular} 
}
\vspace{-0.2in}
\label{tab:acc_bitrate_iid_nonbayesian}
\end{table*}

\textbf{QSGD-KLMS:} As explained in Section~\ref{sec:intro_examples}, \texttt{QSGD} is a stochastic quantization method for FL frameworks that train deterministic model parameters, which outperforms many other baselines. 
Focusing on the most extreme case when the number of quantization levels is $s=1$,  \texttt{QSGD} distribution in (\ref{eq:qsgd}) can be expressed as: 

\begin{align*}
    p_{\texttt{QSGD}}(\mathbf{\hat{v}}^{(t,n)}_i) =
\end{align*}
\begin{align}
\begin{cases}\max \left\{ \frac{ - \mathbf{v}^{(t,n)}_i}{\|\mathbf{v}^{(t,n)}\|}, 0  \right\}    \ &\text{if } \  \mathbf{\hat{v}}^{(t,n)}_i = - \|\mathbf{v}^{(t,n)}\| \\
    \max \left\{ \frac{ \mathbf{v}^{(t,n)}_i}{\|\mathbf{v}^{(t,n)}\|}, 0  \right\}   \ &\text{if } \  \mathbf{\hat{v}}^{(t,n)}_i = \|\mathbf{v}^{(t,n)}\| \\
    1 -  \max \left\{ \frac{ - \mathbf{v}^{(t,n)}_i}{\|\mathbf{v}^{(t,n)}\|}, \frac{ \mathbf{v}^{(t,n)}_i}{\|\mathbf{v}^{(t,n)}\|}, 0  \right\} \ &\text{if} \ \mathbf{\hat{v}}^{(t,n)}_i = 0
    \end{cases}
    \label{eq:qsgd_s1}
\end{align}

which is again a very natural choice for client-only distribution $q_{\phi^{(t,n)}}$ since vanilla \texttt{QSGD} requires the clients to take a sample from $p_{\texttt{QSGD}}(\cdot)$ in (\ref{eq:qsgd_s1}) and communicate the deterministic value of that sample to the server. As for the global distribution, exploiting the temporal correlation in FL, we use the empirical frequencies of the historical updates the server received in the previous round. In other words, in every round $t$, the server records how many clients communicated a negative value ($ - \|\mathbf{v}^{(t,n)}\|$), a positive value ($ \|\mathbf{v}^{(t,n)}\|$), or $0$ per coordinate, and constructs the global distribution $p_{\theta^{(t)}}$ from these empirical frequencies for the next rounds. This new framework, \texttt{QSGD-KLMS}, yields $10$ times reduction in bitrate over vanilla \texttt{QSGD}.

\textbf{SignSGD-KLMS:} Since \texttt{SignSGD}~\citep{bernstein2018signsgd} is not a stochastic quantizer, we first introduce some stochasticity to the vanilla \texttt{SignSGD} algorithm and then integrate $\mathtt{KLMS}$ into it. Instead of mapping the updates to their signs $\pm 1$ deterministically as in vanilla \texttt{SignSGD}, the stochastic version we propose does this mapping by taking a sample from 
\begin{align}
p_{\texttt{SignSGD}}
(\mathbf{\hat{v}}^{(t,n)}_i) =\begin{cases} \sigma ( \frac{ \mathbf{v}^{(t,n)}_i}{M})  \ &\text{if } \ \mathbf{\hat{v}}^{(t,n)}_i = 1 \\
    1- \sigma( \frac{ \mathbf{v}^{(t,n)}_i}{M})  \ &\text{if } \ \mathbf{\hat{v}}^{(t,n)}_i = -1
    \end{cases},
\label{eq:stochastic_sign_sgd}
\end{align}

for some $M>0$, where $\sigma(z) = \frac{1}{1 + e^{-z}}$ is the Sigmoid function. Instead of taking a sample from $p_{\texttt{SignSGD}}(\cdot)$ and sending the deterministic value of the sample by spending $1$ bit per parameter, we can take advantage of the sign symmetry in the model update (about half of the coordinates have positive/negative signs in the update) and reduce the communication cost. For this, we choose $p_{\texttt{SignSGD}}(\cdot)$ in (\ref{eq:stochastic_sign_sgd}) as the client-only distribution $q_{\phi^{(t,n)}}$, and the uniform distribution $U(0.5)$ from the support $\{-1, 1\}$ as the global distribution $p_{\theta^{(t)}}$. This new method, \texttt{SignSGD-KLMS}, achieves higher accuracy than vanilla \texttt{SignSGD} with $\mathbf{66}$ \textbf{times smaller bitrate}.

\textbf{SGLD-KLMS:} From the Bayesian FL family, we focus on the recent \texttt{SGLD} framework \citep{vono2022qlsd} as an example since it provides state-of-the-art results. As discussed in Section~\ref{sec:intro_examples}, due to the stochastic formulation of the Bayesian framework, it is natural to choose the local posterior distributions as the client-only distributions $q_{\phi^{(t,n)}}$, and the global posterior distribution at the server as the global distribution $p_{\theta^{(t)}}$. While extending the existing \texttt{SGLD} algorithm (see Section~\ref{sec:intro_examples}) with \texttt{KLMS}, we inject Gaussian noise locally at each client and scale it such that when all the samples are averaged at the server, the aggregate noise sample $\xi^{(t)}$ (see (\ref{eq:sgld_agg})) is distributed according to $\mathcal{N}(0, \bm{I}_d)$ (more details in Appendix~\ref{app_details_klm_examples}). This new framework, \texttt{SGLD-KLMS}, provides both accuracy and bitrate gains over \texttt{QLSD} \citep{vono2022qlsd} -- the state-of-the-art compression method for Federated \texttt{SGLD}.
\looseness=-2

\section{Experiments}
\label{experiments}

\begin{table*}[t!]
\centering
\caption{\texttt{FedPM-KLMS}, \texttt{QSGD-KLMS}, and \texttt{SignSGD-KLMS} against \texttt{FedPM}, \texttt{QSGD}, 
\texttt{SignSGD},
\texttt{DRIVE}, \texttt{EDEN}, and \texttt{DP-REC} with non i.i.d. split and $20$ out of $100$ clients participating every round. The bitrate is $32$ bpp without any compression. \texttt{FedPM-KLMS} again achieves the highest accuracy with the lowest bitrate -- corresponding to an overall $2,650$ times compression. \texttt{QSGD-KLMS} and \texttt{SignSGD-KLMS} similarly reduce the bitrate up to $56$ times over \texttt{QSGD} and \texttt{SignSGD}.
    }
\resizebox{0.85\linewidth}{!}{
\begin{tabular}{lcccc}
             & \multicolumn{2}{c}{CIFAR-10 (\texttt{CONV6})} & \multicolumn{2}{c}{CIFAR-100 (ResNet-18)} \\
             \toprule
      & \multicolumn{1}{c}{Acc.} & \multicolumn{1}{c}{Bitrate (bpp)}               & \multicolumn{1}{c}{Acc.}  & \multicolumn{1}{c}{Bitrate (bpp)}    \\
      \toprule
\texttt{DRIVE}      & 0.526 &   0.89            &    0.424         &  0.81   \\
\texttt{EDEN}      & 0.528 & 0.89               &    0.425          &    0.81       \\

\texttt{DP-REC}      & 0.530 & 1.08               &    0.424          &  1.00   \\
\midrule
\texttt{QSGD} &    0.552             &     0.140         &  0.429             &   0.150     \\
\texttt{QSGD-KLMS}(ours)           &   \textbf{0.552}           & \textbf{0.071 ($\mathbf{\times 2}$ lower)}       & \textbf{0.429}           &    \textbf{0.072  ($\mathbf{\times 2}$ lower)}  \\
\texttt{QSGD-KLMS}(ours) &        \textbf{0.545}    &      \textbf{0.014  ($\mathbf{\times 10}$ lower)}  & \textbf{0.419}        &   \textbf{0.017    ($\mathbf{\times 8.8}$ lower)}  \\
\midrule
\texttt{SignSGD}      & 0.470 & 1.000               &   0.371          &  0.999   \\
\texttt{SignSGD-KLMS}(ours) &   \textbf{0.530}            & \textbf{0.074  ($\mathbf{\times 14}$ lower)}       & \textbf{0.421 }            &     \textbf{0.044 ($\mathbf{\times 23}$ lower)}   \\
\texttt{SignSGD-KLMS}(ours) &    \textbf{0.520}           &\textbf{0.018 ($\mathbf{\times 56}$ lower)}          & \textbf{0.415}           &  \textbf{0.020  ($\mathbf{\times 50}$ lower)}  \\
\midrule
\texttt{FedPM}      & 0.612  &  0.993             &    0.488          &   0.98       \\
\texttt{FedPM-KLMS}(ours) &    \textbf{0.612}           &    \textbf{0.073 ($\mathbf{\times 12}$ lower)}      &  \textbf{0.488}          &  \textbf{0.074 ($\mathbf{\times 13}$ lower)}   \\
\texttt{FedPM-KLMS}(ours)           &   \textbf{0.599}             &   \textbf{0.016 ($\mathbf{\times 62}$ lower)}       & \textbf{0.480}             &   \textbf{0.012    ($\mathbf{\times 82}$ lower)}\\   
\bottomrule
\end{tabular} 
}

\label{tab:acc_bitrate_noniid_nonbayesian}
\end{table*}

We focus on four $\mathtt{KLMS}$ adaptations we covered in Section~\ref{sec:klm_examples} to empirically demonstrate $\mathtt{KLMS}$'s improvements. We consider four datasets: CIFAR-10 \citep{krizhevsky2009learning}, CIFAR-100 \citep{krizhevsky2009learning}, MNIST \citep{deng2012mnist}, and EMNIST \citep{cohen2017emnist} (with 47 classes). For CIFAR-100, we use ResNet-18~\citep{he2016deep}; for CIFAR-10, a 6-layer CNN \texttt{CONV6}; for MNIST a 4-layer CNN \texttt{CONV4} and LeNet; and for EMNIST, again \texttt{CONV4}. Additional details on the experimental setup and more detailed results with confidence intervals can be found in Appendices~\ref{sec:additional_details_app} and~\ref{sec:additional_results_app}. We first compare \texttt{FedPM-KLMS}, \texttt{QSGD-KLMS}, and \texttt{SignSGD-KLMS} with \texttt{FedPM}~\citep{isik2023sparse}, \texttt{QSGD}~\citep{alistarh2017qsgd}, \texttt{SignSGD}~\citep{bernstein2018signsgd}, \texttt{TernGrad}~\citep{wen2017terngrad}, \texttt{DRIVE}~\citep{vargaftik2021drive}, \texttt{EDEN}~\citep{vargaftik2022eden}, \texttt{MARINA}~\citep{gorbunov2021marina}, \texttt{FedMask}~\citep{li2021fedmask}, and \texttt{DP-REC}~\citep{triastcyn2021dp} on non-Bayesian FL setting in Section~\ref{sec:non_bayesian}. We then provide a comparison of \texttt{SGLD-KLMS} with \texttt{QLSD}~\citep{vono2022qlsd} on the Bayesian FL setting in Section~\ref{sec:bayesian}. Finally, in Section~\ref{sec:ablation}, we present a key ablation study to show how the adaptive block selection strategy in Section~\ref{sec:klm_adaptive} optimizes the bit allocation and helps achieve a smaller bitrate. Clients perform $3$ and $1$ local epochs in the non-Bayesian and Bayesian settings, respectively. We provide multiple (accuracy, bitrate) pairs for \texttt{KLMS} results by varying $D_{KL}^{\text{target}}$. Results are averaged over $3$ runs. 

The codebase is open-sourced at \url{https://github.com/FrancescoPase/Federated-KLMS}.

\subsection{Non-Bayesian Federated Learning} \label{sec:non_bayesian}
\textbf{i.i.d. Data Split:} For the i.i.d. dataset experiments in Table~\ref{tab:acc_bitrate_iid_nonbayesian}, we set the number of clients to $N=10$ and consider full client participation. Table~\ref{tab:acc_bitrate_iid_nonbayesian} shows that \texttt{FedPM-KLMS} and \texttt{SignSGD-KLMS} provide up to $71$ times reduction in communication cost compared to \texttt{FedPM} and \texttt{SignSGD}, respectively (with accuracy boost over vanilla \texttt{SignSGD}). \texttt{QSGD-KLMS}, on the other hand, reduces the communication cost by $10$ times over vanilla \texttt{QSGD}. Overall, \texttt{FedPM-KLMS} requires the smallest bitrate with the highest accuracy among all the frameworks considered. It achieves an \textbf{overall} $\mathbf{2,300}$ \textbf{times compression (compared to 32-bit no-compression case)} by improving the bitrate of vanilla \texttt{FedPM} by $\mathbf{71}$ \textbf{times}. This sets a new standard in the communication-efficient FL literature and marks the significance of side information in FL.  The significant improvements over \texttt{DP-REC} (in both bitrate and accuracy) justify the importance of (i) carefully chosen global and client-only distributions and (ii) the adaptive block selection that optimizes the bit allocation. 

\textbf{Non-i.i.d. Data Split:} For the non-i.i.d. experiments in Table~\ref{tab:acc_bitrate_noniid_nonbayesian}, we compare against \texttt{FedPM}, \texttt{QSGD}, \texttt{SignSGD}, \texttt{DRIVE}, \texttt{EDEN}, and \texttt{DP-REC}. We set the number of clients to $N=100$ and let randomly sampled $20$ of them participate in each round. See Appendix~\ref{sec:additional_details_app} for the non-i.i.d. split strategy. Let $c_{\text{max}}$ be the maximum number of classes each client can see due to the non-i.i.d. split. In the experiments in Table~\ref{tab:acc_bitrate_noniid_nonbayesian}, we set $c_{\text{max}}=40$ for CIFAR-100 and $c_{\text{max}}=4$ for CIFAR-10. See Appendix~\ref{sec:additional_results_app} for other $c_{\text{max}}$ values. Table~\ref{tab:acc_bitrate_noniid_nonbayesian} shows similar gains over the baselines as the i.i.d. experiments in Table~\ref{tab:acc_bitrate_iid_nonbayesian}; in that, \texttt{KLMS} adaptations provide up to $\mathbf{82}$ \textbf{times reduction in the communication cost} compared to the baselines (and $\mathbf{2,650}$ \textbf{times compression compared to 32-bit non-compression case}) with final accuracy as high as (if not higher) the best baseline. This indicates that the statistical heterogeneity level in the data split, while reducing the performance of the underlying training schemes, does not affect the improvement brought by \texttt{KLMS}. We further corroborate this observation with additional experiments in Appendix~\ref{sub:estimation_r_n_het}.

\subsection{Bayesian Federated Learning} \label{sec:bayesian}
Figure~\ref{fig:ablation_adaptve}-(top) compares \texttt{SGLD-KLMS} with \texttt{QLSD}. We consider i.i.d. data split and full client participation with the number of clients $N=10$. It is seen that \texttt{SGLD-KLMS} can reduce the communication cost by $5$ times more than \texttt{QLSD} with higher accuracy on MNIST, where in this case the accuracy is a Monte Carlo average obtained by posterior sampling after convergence. 
\looseness=-1

\begin{figure}[h!]
    \centering
  \includegraphics[width=0.75\columnwidth]{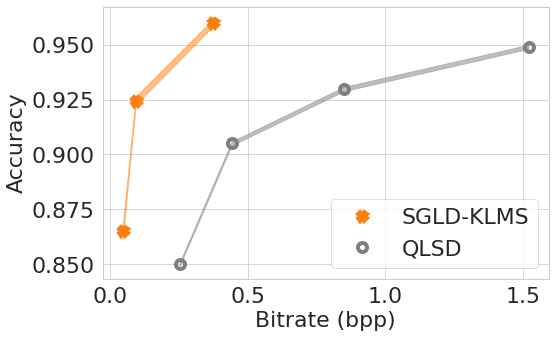}
\includegraphics[width=0.75\columnwidth]{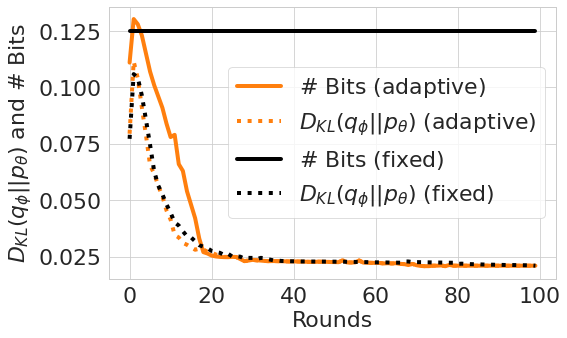}
    \caption{\textbf{(top)} \texttt{SGLD-KLMS} against \texttt{QLSD} using LeNet on i.i.d. MNIST dataset. \textbf{(bottom)} \texttt{FedPM-KLMS (fixed)} against \texttt{FedPM-KLMS (adaptive)} on how well the number of bits approaches the fundamental quantity, KL divergence -- using \texttt{CONV6} on i.i.d. CIFAR-10. Both KL divergence and the number of bits are normalized by the number of parameters.  
    }
    \label{fig:ablation_adaptve}
\end{figure}

\looseness=-1
\subsection{Ablation Study: The Effect of the Adaptive Bit Allocation Strategy } 
\label{sec:ablation}
We conduct an ablation study to answer the following question: \emph{Does adaptive bit allocation strategy really help optimize the bit allocation and reduce $\#$ bits down to KL divergence?} To answer this question, in Figure~\ref{fig:ablation_adaptve}-(bottom), we show how the average per-parameter KL divergence and $\#$ bits spent per parameter change over the rounds for \texttt{FedPM-KLMS} with fixed- and adaptive-size blocks. We adjust the hyperparameters such that the final accuracies differ by only $0.01\%$ on CIFAR-10. For the fixed-size experiments, since we fix $K$ (number of samples per block) and the block size for the whole model and across all rounds, $\#$ bits per parameter stays the same while the KL divergence shows a decreasing trend. On the other hand, in the adaptive-size experiments, the block size changes across the model parameters and the rounds to guarantee that each block has the same KL divergence. Since all blocks have the same KL divergence, we spend the same $\#$ bits for each block (with adaptive size) as suggested by Theorem~\ref{thm2_main}, which adaptively optimizes the bitrate towards the KL divergence. This is indeed justified in Figure~\ref{fig:ablation_adaptve}-(bottom) since the $\#$ bits curve quickly approaches the KL divergence curve.
\looseness=-1

\section{Discussion \& Conclusion}
\label{conclusion}

We leveraged side information that is naturally present in many existing FL frameworks to reduce the bitrate by up to $\mathbf{82}$ \textbf{times over the baselines} without an accuracy drop. This corresponds to an \textbf{overall} $\textbf{2,650}$ \textbf{times compression} compared to the no-compression case. We discovered highly natural choices for side information (global distribution at the server) in popular stochastic FL frameworks without requiring any change, i.e., the side information naturally arises in the original frameworks. 
We believe the proposed way of using side information will set a new standard in communication-efficient FL as it can provide similar bitrate reduction in many FL frameworks. 

\section{Acknowledgements}
\label{acknowledgement}

The authors would like to thank Saurav Chatterjee for his help in the theoretical extension of \citet{chatterjee2018sample} to the $N$-user case; also Yibo Zhang, Xiaoyang Wang, Enyi Jiang, and Brando Miranda for their feedback on an earlier draft. This work is partially supported by a Google Ph.D. Fellowship, a Stanford Graduate Fellowship, NSF III 2046795, IIS 1909577, CCF 1934986, NIH 1R01MH116226-01A, NIFA award 2020-67021-32799, the Alfred P. Sloan Foundation, Meta, Google, and the European Union under the Italian National Recovery and Resilience Plan (NRRP) of NextGenerationEU, partnership on “Telecommunications of the Future” (PE0000001 - program ''RESTAR``).

%\clearpage
% If you use natbib package, activate the following three lines:
%\usepackage[round]{natbib}
%\renewcommand{\bibname}{References}
%\renewcommand{\bibsection}{\subsubsection*{\bibname}}

% If you use BibTeX in apalike style, activate the following line:
%\bibliographystyle{apalike}
\bibliography{refs}

\clearpage
\onecolumn
\appendix
\appendix
\section{Additional Details on Prior Work}
\label{appendix_prior}

\subsection{\texttt{FedPM} \citep{isik2023sparse}} \label{appendix:fedpm}
We provide the pseudocode for \texttt{FedPM} in Algorithms~\ref{algorithm:fedpm} and~\ref{algorithm:aggregation}. See \citep{isik2023sparse} for more details.

\begin{algorithm}[!h]
        {\bf Hyperparameters:} local learning rate $\eta_L$, minibatch size $B$, number of local iterations $\tau$.\\
        {\bf Inputs:} local datasets $\mathcal{D}_i$, $i=1, \dots, N$, number of iterations $T$. \\
        {\bf Output:} random  \texttt{SEED} and binary mask parameters $\bm{m}^{\text{final}}$.\\
    %\vspace{-.2in}
        \begin{algorithmic}
            \STATE{At the server, initialize a random network with weight vector $\bm{w}^{\text{init}} \in \mathbb{R}^d$ using a random \texttt{SEED}, and broadcast it to the clients. }
            \STATE{At the server, initialize the random score vector $\bm{s}^{(0, g)} \in \mathbb{R}^d$, and compute $\theta^{(0,g)} \gets \text{Sigmoid}(\bm{s}^{(0, g)})$.}
            \STATE{At the server, initialize Beta priors $\bm{\alpha}^{(0)} = \bm{\beta}^{(0)} = \bm{\lambda}_0$.}
            \FOR{$t=1, \dots, T$} 
            \STATE{Sample a subset $\mathcal{C}_t \subset \{1, \dots, N\}$ of $|\mathcal{C}_t|=C$ clients without replacement.}
            \STATE{\textbf{On Client Nodes:}}
            \FOR{$c \in \mathcal{C}_t$}
            \STATE{Receive $\theta^{(t-1, g)}$ from the server and set $\bm{s}^{(t, c)} \gets \text{Sigmoid}^{-1}(\theta^{(t-1, g)})$.}
            \FOR{$l=1, \dots, \tau$}
            \STATE{$\phi^{(t, c)} \gets \text{Sigmoid}(\bm{s}^{(t, c)})$ }
            \STATE{Sample binary mask $\bm{m}^{(t, c)} \sim q_{\bm{m}^{(t, c)}} = \text{Bern}(\phi^{(t, c)})$. }
            \STATE{$\dot{\bm{w}}^{(t, c)} \gets \bm{m}^{(t, c)} \odot \bm{w}^{\text{init}}$ }
            \STATE{$g_{\bm{s}^{(t, c)}} \gets \frac{1}{B} \sum_{b=1}^{B} \nabla \ell(\dot{\bm{w}}^{(t, c)}; \mathcal{S}_b^{c})$; where  $\{ \mathcal{S}_b^{c}\}_{b=1}^{B}$ are uniformly chosen from $\mathcal{D}_c$}
            \STATE{$\bm{s}^{(t, c)} \gets \bm{s}^{(t, c)} - \eta_L \cdot g_{\bm{s}^{(t, c)}}$}
            \ENDFOR \\
            \STATE{$\phi^{(t, c)} \gets \text{Sigmoid}(\bm{s}^{(t, c)})$}
            \STATE{Sample a binary mask $\bm{m}^{(t, c)} \sim \text{Bern}(\phi^{(t, c)})$.}
            \STATE{Send the arithmetic coded binary mask $\bm{m}^{(t, c)}$ to the server.}
            \ENDFOR \\
            \STATE
            \STATE{\textbf{On the Server Node:}}
            \STATE{Receive $\bm{m}^{(t, c)}$'s from $C$ client nodes.}
            \STATE{$\theta^{(t, g)} \gets $ BayesAgg( $\{\bm{m}^{(t, c)}\}_{c \in \mathcal{C}_t}$, $t$) $\quad$ // See Algorithm~\ref{algorithm:aggregation}.}
            \STATE{Broadcast $\theta^{(t, g)}$ to all client nodes.}
            \ENDFOR \\
            \STATE{Sample the final binary mask $\bm{m}^{\text{final}} \sim \text{Bern}(\theta^{(T, g)})$.}
            \STATE{Generate the final model: $\dot{\bm{w}}^{\text{final}} \gets \bm{m}^{\text{final}} \odot \bm{w}^{\text{init}}$.}
            \end{algorithmic}
        \caption{Federated Probablistic Mask Training (\texttt{FedPM}) \citep{isik2023sparse}.}
        \label{algorithm:fedpm}
    % \vspace{-1mm}
    \end{algorithm}

\begin{algorithm}[!h]
        {\bf Inputs:} clients' updates $\{\bm{m}^{(t, c)}\}_{c \in \mathcal{C}_t}$, and round number $t$ \\
        {\bf Output:} global probability mask $\bm{\pi}^{(t)}$\\
        \begin{algorithmic}
        \IF{ResPriors($t$)}
        \STATE{$\bm{\alpha}^{(t-1)} \gets \bm{\beta}^{(t-1)} = \bm{\lambda}_0 $}
        \ENDIF \\
        \STATE{Compute $\bm{m}^{(t, \text{agg})} = \sum_{k \in \mathcal{C}_t} \bm{m}^{(t, c)}$.}
        \STATE{$\bm{\alpha}^{(t)} \gets \bm{\alpha}^{(t-1)} + \bm{m}^{(t, \text{agg})}$}
        \STATE{$\bm{\beta}^{(t)} \gets \bm{\beta}^{(t-1)} + C \cdot \bm{1} - \bm{m}^{(t, \text{agg})}$}
        \STATE{$\bm{\pi}^{(t)} \gets \frac{\bm{\alpha}^{(t - 1)}}{\bm{\alpha}^{(t)} + \bm{\beta}^{(t)} - 2}$}
        \STATE{Return $\bm{\pi}^{(t)}$}
        \end{algorithmic}
\caption{BayesAgg. \citep{isik2023sparse}}
        \label{algorithm:aggregation}
\end{algorithm}

\subsection{\texttt{QSGD} \citep{alistarh2017qsgd}} \label{appendix:qsgd}
We provide the pseudocode for \texttt{QSGD} in Algorithm~\ref{algorithm:qsgd}. See \citep{alistarh2017qsgd} for more details.

\begin{algorithm}[!h]
        {\bf Hyperparameters:} server learning rate $\eta_S$, local learning rate $\eta_L$, number of quantization levels $s$, minibatch size $B$.\\
        {\bf Inputs:} local datasets $\mathcal{D}_n$, $n=1, \dots, N$, number of iterations $T$.\\
        {\bf Output:} final model $\bm{w}^{(T)}$.
    %\vspace{-.2in}
        \begin{algorithmic}
            \STATE{At the server, initialize a random network with weight vector $\bm{w}^{(0,g)} \in \mathbb{R}^d$ and broadcast it to the clients. }
            \FOR{$t=1, \dots, T$} 
            \STATE{Sample a subset $\mathcal{C}_t \subset \{1, \dots, N\}$ of $|\mathcal{C}_t| = C$ clients without replacement.}
            \STATE{\textbf{On Client Nodes:}}
            \FOR{$c \in \mathcal{C}_t$}
            \STATE{Receive $\bm{w}^{(t-1, g)}$ from the server and set the local model parameters $\bm{w}^{(t, c)} \gets \bm{w}^{(t, g)}$.}
            \FOR{$l=1, \dots, \tau $}
            \STATE{$g_w^{(t, c)} \gets\frac{1}{B} \sum_{b=1}^{B} \nabla \ell(\bm{w}^{(t, c)}; \mathcal{S}_b^{c})$; where  $\{ \mathcal{S}_b^{c}\}_{b=1}^{B}$ are uniformly chosen from $\mathcal{D}_c$} 
            \STATE{$\bm{w}^{(t, c)} \gets \bm{w}^{(t, c)} - \eta_L \cdot g_w^{(t, c)}$}
            \ENDFOR \\
        \STATE{$\mathbf{v}^{(t,c)} \gets \bm{w}^{(t,c)} - \bm{w}^{(t, g)}$}
            \FOR{$i = 1, \dots, d$}
            \STATE{Find integer $0 \leq q \leq s$ such that $|\mathbf{v}^{(t,c)}_i| / \|\mathbf{v}^{(t,c)}\|_2 \in [ q/s, (q+1)/s]$.}
            \STATE{Take a sample $z \sim \text{Bern}(1 - (\frac{|\mathbf{v}^{(t, c)}_i|}{\|\mathbf{v}^{(t,c)}\|_2} s -q))$.}
            \IF{$z = 1$}
            \STATE{$\mathbf{\kappa}^{(t,c)}_i \gets q/s$.}
            \ELSE
            \STATE{$\mathbf{\kappa}^{(t,c)}_i \gets (q+1)/s$.}
            \ENDIF
            \ENDFOR \\
            \STATE{Send vectors $\mathbf{\kappa}^{(t,c)}$, $\text{sign}(\mathbf{v}^{(t,c)})$, and norm $\|\mathbf{v}^{(t,n)}\|_2$ to the server using Elias coding~\citep{elias1975universal} as in~\citep{alistarh2017qsgd}.}
            \ENDFOR \\
            \STATE
            \STATE{\textbf{On the Server Node:}}
            \STATE{Receive $\mathbf{\kappa}^{(t,c)}$, $\text{sign}(\mathbf{v}^{(t,c)})$, and norm $\|\mathbf{v}^{(t,c)}\|_2$ from the clients $c \in \mathcal{C}_t$.}
            \FOR{$c \in \mathcal{C}_t$}
            \FOR{$i = 1, \dots, d$}
            \STATE{Reconstruct $\mathbf{\hat{v}}^{(t,c)}_i \gets \|\mathbf{v}^{(t,c)}\|_2 \cdot \text{sign}(\mathbf{v}^{(t,c)}_i) \cdot \mathbf{\kappa}^{(t,c)}_i$}.
            \ENDFOR
            \ENDFOR
            \STATE{Aggregate and update $\bm{w}^{(t, g)} \gets \bm{w}^{(t-1, g)} - \eta_S \frac{1}{C} \sum_{c \in \mathcal{C}_t} \mathbf{\hat{v}}^{(t, c)}$.}
            \STATE{Broadcast $\bm{w}^{(t, g)}$ to the clients.}
            \ENDFOR \\
            \end{algorithmic}
        \caption{Quantized Stochastic Gradient Descent (\texttt{QSGD}) \citep{alistarh2017qsgd}.}
        \label{algorithm:qsgd}
    % \vspace{-1mm}
    \end{algorithm}

\newpage

\subsection{\texttt{QLSD} \citep{vono2022qlsd}} \label{appendix:qlsd}
We provide the pseudocode for \texttt{QLSD} in Algorithm~\ref{algorithm:qlsd}. See \citep{vono2022qlsd} for more details.

\begin{algorithm}[!h]
        {\bf Hyperparameters:} server learning rate $\eta_S$, number of quantization levels $s$, minibatch size $B$.\\
        {\bf Inputs:} local datasets $\mathcal{D}_n$, $n=1, \dots, N$, number of iterations $T$.\\
        {\bf Output:} samples $\left\{ \theta^{(t)}\right\}_{t=1}^T.$\\
    %\vspace{-.2in}
        \begin{algorithmic}
            \STATE{At the server, initialize a random network with weight vector $\theta^{(0)} \in \mathbb{R}^d$ and broadcast it to the clients. }
            \FOR{$t=1, \dots, T$} 
            \STATE{Sample a subset $\mathcal{C}_t \subset \{1, \dots, N\}$ of $|\mathcal{C}_t|=C$ clients without replacement.}
            \STATE{\textbf{On Client Nodes:}}
            \FOR{$c \in \mathcal{C}_t$}
            \STATE{Receive $\theta^{(t-1)}$ from the server and set the local model parameters $\phi^{(t, c)} \gets \theta^{(t-1)}$.}
            \STATE{Sample a minibatch $\mathcal{S}^c$ s.t. $\left|\mathcal{S}^c \right| = B$ uniformly from $\mathcal{D}_c$.}
            \STATE{Compute a stochastic gradient of the potential $H(\phi^{(t, c)}) \gets \frac{|D_j^{(c)}|}{B} \sum_{j \in \mathcal{S}^{c}} \nabla U_j(\phi^{(t, c)})$.}
            \FOR{$i = 1, \dots, d$}
            \STATE{Find integer $0 \leq q \leq s$ such that $\frac{\left|H_i(\phi^{(t, c)})\right|}{\|H(\phi^{(t, c)})\|_2} \in [ q/s, (q+1)/s]$.}
            \STATE{Take a sample $z \sim \text{Bern}(1 - (\frac{|H_i(\phi^{(t, c)})|}{\|H(\phi^{(t, c)})\|_2} s -q))$.}
            \IF{$z = 1$}
            \STATE{$\mathbf{\kappa}^{(t,c)}_i \gets q/s$.}
            \ELSE
            \STATE{$\mathbf{\kappa}^{(t,c)}_i \gets (q+1)/s$.}
            \ENDIF
            \ENDFOR \\
            \STATE{Send vectors $\mathbf{\kappa}^{(t,c)}$, $\text{sign}(H(\phi^{(t, c)}))$, and norm $\|H(\phi^{(t, c)})\|_2$ to the server using Elias coding~\citep{elias1975universal} as in~\citep{alistarh2017qsgd}.}
            \ENDFOR \\
            \STATE
            \STATE{\textbf{On the Server Node:}}
            \STATE{Receive $\mathbf{\kappa}^{(t,c)}$, $\text{sign}(H(\phi^{(t, c)}))$, and norm $\|H(\phi^{(t, c)})\|_2$ from the clients $c \in \mathcal{C}_t$.}
            \FOR{$c \in \mathcal{C}_t$}
            \FOR{$i = 1, \dots, d$}
            \STATE{Reconstruct $\hat{H}_i(\phi^{(t, c)}) \gets \|H(\phi^{(t, c)})\|_2 \cdot \text{sign}(H_i(\phi^{(t, c)})) \cdot \mathbf{\kappa}^{(t,c)}_i$}.
            \ENDFOR
            \ENDFOR
            \STATE{Compute $\hat{H}(\phi^{(t)}) \gets \frac{N}{C} \sum_{c \in \mathcal{C}_t} \hat{H}(\phi^{(t, c)})$.}
            \STATE{Sample $\xi^{(t)} \sim \mathcal{N}(\bm{0}_d, \bm{I}_d)$.}
            \STATE{Compute $\theta^{(t)} \gets \theta^{(t-1)} - \eta_S \hat{H}(\phi^{(t)}) + \sqrt{2\gamma} \xi^{(t)}$.}
            \STATE{Broadcast $\theta^{(t)}$ to the clients.}
            \ENDFOR \\
            \end{algorithmic}
        \caption{Quantised Langevin Stochastic Dynamics (\texttt{QLSD}) \citep{vono2022qlsd}.}
        \label{algorithm:qlsd}
    % \vspace{-1mm}
    \end{algorithm}

\newpage

\section{\texttt{KLMS} Pseudocode} \label{app_klm}
In this section, we provide pseudocodes for both versions of \texttt{KLMS}: Algorithm~\ref{algorithm:fixed_sample_index} with fixed-sized blocks (\texttt{Fixed-KLMS}), and Algorithm~\ref{algorithm:adaptive_sample_index} with adaptive-sized blocks (\texttt{Adaptive-KLMS}). The algorithms are standalone coding modules that can be applied to different FL frameworks (see Appendix~\ref{app_details_klm_examples}). In the experiments in Section~\ref{experiments}, we used \texttt{Adaptive-KLMS} and called it \texttt{KLMS} for simplicity. The decoding approach at the server is outlined in Algorithm~\ref{algorithm:klms_decoder}.

\begin{algorithm}[!h]
        {\bf Inputs:} client-only $q_{\phi^{(t, c)}}$ and global $p_{\theta^{(t)}}$ distributions, block size $S$, number of per-block samples $K$. \\
        {\bf Output:} selected indices for each block $\{k^{(c)*}_{[m]}\}_{m=1}^M$, where $M = \lceil \frac{d}{S} \rceil$ is the number of bloks.\\
        \begin{algorithmic}
        \STATE{Define $\{ q_{\phi^{(t, c)}_{[m]}} \}_{m=1}^{M}$ and $\{ p_{\theta^{(t, c)}_{[m]}} \}_{m=1}^{M}$ splitting $q_{\phi^{(t, c)}}$ and $p_{\theta^{(t)}}$ into $M$ distributions on $S$-size parameters blocks.}
        \FORALL{$m \in \{1, \dots, M\}$}
        \STATE{$I \gets [(m-1) S: mS]$.}
        \STATE{Take $K$ samples from the global distribution: $\{\mathbf{y}_{[k]}\}_{k=1}^K \sim p_{\theta^{(t)}_{[I]}}$.}
        \STATE{$\alpha_{[k]} \gets \frac{q_{\phi^{(t, c)}_{[I]}}\left(\mathbf{y}_{[k]}\right)}{p_{\theta^{(t)}_{[I]}}\left(\mathbf{y}_{[k]}\right)}$ $\forall k \in \{ 1, \dots, K\}$.}
        \STATE{$\pi(k) \gets \frac{\alpha_{[k]}}{\sum_{k'=1}^K \alpha_{[k']}}$ $\forall k \in \{ 1, \dots, K\}$.}
        \STATE{Sample an index $k_{[m]}^{(c)*} \sim \pi(k)$.}
        \ENDFOR
        \STATE{Send the selected indices $\{k^{(c)*}_{[m]}\}_{m=1}^M$ with $ M \cdot \log_2 K$ bits overall for $M$ blocks.}
        \end{algorithmic}
\caption{\texttt{Fixed-KLMS}.}
        \label{algorithm:fixed_sample_index}
\end{algorithm}

\begin{algorithm}[!h]
        {\bf Inputs:} client-only $q_{\phi^{(t, c)}}$ and global $p_{\theta^{(t)}}$ distributions, block locations $M$ (a list of start indices of each block), number of per-block samples $K$, target KL divergence $D_{KL}^{\text{target}}$, the flag \texttt{UPDATE} indicating whether the block locations will be updated, the maximum block size allowed \texttt{MAX\_BLOCK\_SIZE}. \\
        {\bf Output:} selected indices for each block $\{k^{(c)*}_{[m]}\}_{m=1}^M$, where the number of blocks $M$ may vary each round.  \\
        \begin{algorithmic}
        \IF{$\texttt{UPDATE}$}
        \STATE{Construct the sequence of per-coordinate KL-divergence of size $d$: $\mathbf{D} \gets \left [ D_{KL}(q_{\phi^{(t, c)}_1} \| p_{\theta^{(t)}_1}), D_{KL}(q_{\phi^{(t, c)}_2} \| p_{\theta^{(t)}_2}), \dots, D_{KL}(q_{\phi^{(t, c)}_d} \| p_{\theta^{(t)}_d})\right ]$.}
        \STATE{Divide $\mathbf{D}$ into subsequences of $\{ \mathbf{D}[i_1=1:i_2], \mathbf{D}[i_2:i_3], \dots, \mathbf{D}[i_{M}:i_{M+1}=d] \}$ such that for all $m=1, \dots, M$,  $\sum_{l=i_{m}}^{i_{m+1}}\mathbf{D}[l] \approx D_{KL}^{\text{target}}$ or $i_{m+1}- i_{m} = $ \texttt{MAX\_BLOCK\_SIZE}. Here $M$, i.e, the number of blocks, may vary each round.}
        \STATE Construct new block locations: $I_m \gets [i_m: i_{m+1}]$ for $m=1, \dots, M$.
        \ELSE
        \STATE{Keep the old block locations $I$.}
        \ENDIF
        \STATE{Construct per-block client-only $\{ q_{\phi^{(t, c)}_{[I_m]}} \}_{m=1}^{M}$ and global $\{ p_{\theta^{(t)}_{[I_m]}} \}_{m=1}^{M}$ distributions.}
        \FORALL{$m \in \{1, \dots, M\}$}
        \STATE{Sample $\{\mathbf{y}_{[k]}\}_{k=1}^K \sim p_{\theta^{(t)}_{[I_m]}}$.}
        \STATE{$\alpha_{[k]} \gets \frac{q_{\phi^{(t, c)}_{[I_m]}}\left(\mathbf{y}_{[k]}\right).}{p_{\theta^{(t)}_{[I_m]}}\left(\mathbf{y}_{[k]}\right)}$ $\forall k \in \{ 1, \dots, K\}$.}
        \STATE{$\pi (k) \gets \frac{\alpha_{[k]}}{\sum_{k'=1}^K \alpha_{[k']}}$ $\forall k \in \{ 1, \dots, K\}$.}
        \STATE{Sample $k_{[m]}^{(c)*} \sim \pi (k)$.}
        \ENDFOR
        \IF{\texttt{UPDATE}}
        \STATE{Return the selected indices $\{k^{(c)*}_{[m]}\}_{m=1}^M$ and the new block locations $I$ spending $\approx D_{KL}^{\text{target}} + \log_2($\texttt{MAX\_BLOCK\_SIZE}$)$ bits per block (block sizes are different for each block).}
        \ELSE
        \STATE{Return the selected indices $\{k^{(c)*}_{[m]}\}_{m=1}^M$ spending $\approx D_{KL}^{\text{target}}$ bits per block (block sizes are different for each block).}
        \ENDIF
        \end{algorithmic}
\caption{\texttt{Adaptive-KLMS}.}
        \label{algorithm:adaptive_sample_index}
\end{algorithm}

\begin{algorithm}[!h]
        {\bf Inputs:} client block locations $\{ I^{(t, c)}\}_{c \in \mathcal{C}_t}$.\\
        {\bf Output:} new global block locations $I^{(t)}$.\\
        \begin{algorithmic}
        \STATE{Define empty $I^{(t)}$.}
            \STATE{$m_{\text{max}} \gets \max_{c \in \mathcal{C}_t} \left\{ \text{length}(I^{(t, c)})\right\}$.}
            \FOR{$m \in \{1, 2, \dots, m_{\text{max}}\}$}
                \STATE{$\tilde{i}_m \gets 0$.}
                \STATE{$l \gets 0$.}
                \FOR{$ c \in \mathcal{C}_t$}
                    \IF{$\text{length}(I^{(t, c)}) \geq m$}
                        \STATE{$\tilde{i}_m  \gets \tilde{i}_m + I^{(t, c)}_{i_m}$.}
                        \STATE{$l \gets l + 1$.}
                    \ENDIF
                \ENDFOR
                \STATE{$\Bar{i}_m \gets \lceil \tilde{i}_m / l \rceil$.}
                \STATE{Add $\Bar{i}_m$ to $I^{(t)}$.}
            \ENDFOR
        \STATE{Return $I^{(t)}$.}
        \end{algorithmic}
\caption{\texttt{Aggregate-Block-Locations}.}
        \label{algorithm:aggregate_indices}
\end{algorithm}

\begin{algorithm}[!h]
        {\bf Inputs:}  global $p_{\theta^{(t)}}$ distribution, block locations $I$ of $M$ blocks, number of per-block samples $K$, selected indices for each block $\{k^{(c)*}_{[m]}\}_{m=1}^M$, where $M = \lceil \frac{d}{S} \rceil$ is the number of blocks. \\
        {\bf Output:} The selected samples $\{\mathbf{y}^{*}_{[m]}\}_{m=1}^M$ for each block. \\
        \begin{algorithmic}
        \STATE{Define $\{ p_{\theta^{(t)}_{[I_m]}} \}_{m=1}^{M}$ splitting $p_{\theta^{(t)}}$ into $M$ distributions with block locations in $I$.}
        \FORALL{$m \in \{1, \dots, M\}$}
        \STATE{Take $K$ samples from the global distribution: $\{\mathbf{y}_{[k]}\}_{k=1}^K \sim p_{\theta^{(t)}_{[I_m]}}$.}
        \STATE{Recover $\mathbf{y}^{*}_{[m]} \gets \mathbf{y}_{k^{(c)*}_{[m]}}$.  (Recall that $k^{(c)*}_{[m]}$ for each block $m$ was received from the client.)}
         \ENDFOR
        \STATE{Return the selected samples $\{\mathbf{y}^{*}_{[m]}\}_{m=1}^M$ for each block.}
        \end{algorithmic}
\caption{\texttt{KLMS}-Decoder.}
        \label{algorithm:klms_decoder}
\end{algorithm}

\newpage

\clearpage
\section{Proofs}
\label{appendix_proofs}
In this section, we provide the proof for Theorem~\ref{thm2_main}. But before that, we first define the formal problem statement, introduce some new notation, and give another theorem (Theorem~\ref{thm1}) that will be required for the proof of Theorem~\ref{thm2_main}.

We consider a scenario where $N$ distributed nodes and a centralized server share a prior distribution $p_{\theta}$ over a set $\mathcal{X}$ equipped with some sigma algebra. Each node $n$ also holds a posterior distribution $q_{\phi^{(n)}}$ over the same set. The server wants to estimate $\mathbb{E}_{X^{(n)} \sim q_{\phi^{(n)}} \forall n \in [N]} [\frac{1}{N} \sum_{m=1}^N f(X^{(m)})]$, where $f(\cdot): \mathcal{X} \rightarrow \mathbb{R}$ is a measurable function. In order to minimize the cost of communication from the nodes to the centralized server, each node $n$ and the centralized server take $K^{(n)}$ samples from the prior distribution $\mathbf{y}^{(n)}_{[1]}, \dots, \mathbf{y}^{(n)}_{[K^{(n)}]} \sim p_{\theta}$. Then client $n$ performs the following steps:
\begin{enumerate}

    \item Define a new probability distribution over the indices $k=1, \dots, K^{(n)}$:

        \begin{align}
        \pi^{(n)}(k) = \frac{q_{\phi^{(n)}}(\mathbf{y}^{(n)}_{[k]}) / p_{\theta}(\mathbf{y}^{(n)}_{[k]})}{\sum_{l=1}^{K^{(n)}} q_{\phi^{(n)}}(\mathbf{y}^{(n)}_{[l]}) / p_{\theta}(\mathbf{y}^{(n)}_{[l]})}  
    \end{align}

    and over the samples $\mathbf{y}^{(n)}_{[1]}, \dots, \mathbf{y}^{(n)}_{[K^{(n)}]}$:
    \begin{align}
        \tilde{q}_{\pi^{(n)}}(\mathbf{y}) = \sum_{k=1}^{K^{(n)}}\pi^{(n)}(k) \cdot \mathbf{1}(\mathbf{y}_{[k]}^{(n)} = \mathbf{y}).
    \end{align}
    \item Sample $k^{(n)*} \sim \pi^{(n)}$.
    \item Communicate $k^{(n)*}$ to the centralized server with $\log K^{(n)}$ bits.
\end{enumerate}

Then, the centralized server recovers the sample $\mathbf{y}^{(n)}_{[k^{(n)*}]}$ that it generated in the beginning. (Note that $\mathbf{y}^{(n)}_{[k^{(n)*}]}$ is actually a sample from $ \tilde{q}_{\pi^{(n)}}$.) Finally, the server aggregates these samples $\frac{1}{N} \sum_{n=1}^N f(\mathbf{y}^{(n)}_{k^{(n)*}})$ which is an estimate of

\begin{align}
    \mathbb{E}_{Y^{(n)} \sim \tilde{q}_{\pi^{(n)}} \forall n \in [N]} [\frac{1}{N} \sum_{m=1}^N f(Y^{(m)})].
\end{align}

We want to find a relation between the number of samples $K^{(1)}, \dots, K^{(N)}$ (or the number of bits $\log K^{(1)}, \dots, \log K^{(N)}$) and the error in the estimate, $|\mathbb{E}_{Y^{(n)} \sim \tilde{q}_{\pi^{(n)}} \forall n \in [N]} [\frac{1}{N} \sum_{m=1}^N f(Y^{(m)})] -\mathbb{E}_{X^{(n)} \sim q_{\phi^{(n)}} \forall n \in [N]} [\frac{1}{N} \sum_{m=1}^N f(X^{(m)})]|$. In our proofs, we closely follow the methodology in Theorems 1.1. and 1.2. in \citep{chatterjee2018sample}. In Theorem~\ref{thm1}, we use the probability density of $q_{\phi^{(n)}}$ with respect to $p_{\theta}$ for each node $n$ and denote it by $\rho_n = \frac{d q_{\phi^{(n)}}}{dp_{\theta}}$. We refer to the following definitions often:

\begin{align}
    I(f) &= \int_{\mathbf{x}^{(1)}} \dots \int_{\mathbf{x}^{(N)}} \left ( \frac{1}{N} \sum_{n=1}^N f(\mathbf{x}^{(n)}) \right ) \prod_{n=1}^N dq_{\phi^{(n)}}(\mathbf{x}^{(n)}),
 \end{align}

 \begin{align}
        I_K(f) &= \frac{1}{\prod_{n=1}^N K^{(n)}}\sum_{k^{(1)}=1}^{K^{(1)}} \dots \sum_{k^{(N)}=1}^{K^{(N)}} \left ( \frac{1}{N}\sum_{n=1}^N f(\mathbf{y}^{(n)}_{[k^{(n)}]})   \right )\prod_{n=1}^N \rho_n (\mathbf{y}^{(n)}_{[k^{(n)}]}),
 \end{align}
 and 

  \begin{align}
        J_K(f) &= \sum_{k^{(1)}=1}^{K^{(1)}} \dots \sum_{k^{(N)}=1}^{K^{(N)}} \left ( \frac{1}{N}\sum_{n=1}^N f(\mathbf{y}^{(n)}_{[k^{(n)}]})  \right ) \prod_{n=1}^N \frac{q_{\phi^{(n)}}(\mathbf{y}^{(n)}_{[k^{(n)}]}) / p_{\theta}(\mathbf{y}^{(n)}_{[k^{(n)}]})}{\sum_{l=1}^{K^{(n)}} q_{\phi^{(n)}}(\mathbf{y}^{(n)}_{[l]}) / p_{\theta}(\mathbf{y}^{(n)}_{[l]})}.  \label{def:JK}
 \end{align}

 Notice that $I(f)$ corresponds to the target value the centralized server wants to estimate, $J_K(f)$ is the estimate from the proposed approach, and $I_K(f)$ is a value that will be useful in the proof and that satisfies $\mathbb{E} [I_K(f)] = I(f)$.

\begin{theorem} \label{thm1}
Let $p_{\theta}$ and $q_{\phi^{(n)}}$ for $n=1, \dots, N$ be probability distributions over a set $\mathcal{X}$ equipped with some sigma-algebra. Let $X^{(n)}$ be
an $\mathcal{X}$-valued random variable with law $q_{\phi^{(n)}}$. Let $r \geq 0$ and $\tilde{q}_{\pi^{(n)}}$ for $n=1, \dots, N$ be discrete distributions each constructed by $K^{(n)}= \exp{ \left (D_{KL}(q_{\phi^{(n)}} \| p_{\theta} )  + r  \right )}$ samples $\{\mathbf{y}^{(n)}_{[k^{(n)}]}\}_{k^{(n)}=1}^{K^{(n)}}$ from $p_{\theta}$ defining $\tilde{q}_{\pi^{(n)}}(\mathbf{y}) = \sum_{k=1}^{K^{(n)}} \frac{q_{\phi^{(n)}}(\mathbf{y}^{(n)}_{[k]}) / p_{\theta}(\mathbf{y}^{(n)}_{[k]})}{\sum_{l=1}^{K^{(n)}} q_{\phi^{(n)}}(\mathbf{y}^{(n)}_{[l]}) / p_{\theta}(\mathbf{y}^{(n)}_{[l]})} \cdot \mathbf{1}(\mathbf{y}^{(n)}_{[k]} = \mathbf{y})$. Furthermore, for $f(\cdot)$ defined above, let $||f||_{\mathbf{q_{\phi}}} = \sqrt{\mathbb{E}_{X^{(n)} \sim q_{\phi^{(n)}} \forall n \in [N]} [ (\frac{1}{N}\sum_{m=1}^N f(X^{(m)}))^2] }$ be its 2-norm under $\mathbf{q_{\phi}} = q_{\phi^{(1)}}, \dots, q_{\phi^{(N)}}$. Then, 

\begin{align}
    \mathbb{E} |I_K(f) - I(f)| \leq ||f||_{\mathbf{q_{\phi}}} \left (e^{-Nr/4} + 2 \sqrt{\prod_{n=1}^N \mathbb{P} \left ( \log \rho_n(X^{(n)}) >  D_{KL} (q_{\phi^{(n)}} || p_{\theta}) + r/2 \right )} \right ).
\end{align}

Conversely, let $\mathbf{1}$ denote the function from $\mathcal{X}$ into $\mathbb{R}$ that is identically equal to $1$. If for $n=1, \dots, N$, $K^{(n)} = \exp{ \left (D_{KL}(q_{\phi^{(n)}} || p_{\theta} )  - r  \right ) }$ for some $r \geq 0$, then for any $\delta \in (0,1)$,

\begin{align}
    \mathbb{P}(I_K(\mathbf{1}) \geq 1 - \delta) \leq e^{-  Nr/2} + \frac{\prod_{n=1}^N 
 \mathbb{P} \left ( \log \rho_n(X^{(n)}) \leq   D_{KL}(q_{\phi^{(n)}} || p_{\theta}) - r/2 \right )}{1- \delta}.
\end{align}

\end{theorem}

\begin{proof}
Let $L^{(n)} = D_{KL}(q_{\phi^{(n)}} || p_{\theta} ), \forall n \in [N]$. Suppose that $K^{(n)}=e^{L^{(n)}+r}$ and $a^{(n)} = e^{L^{(n)}+ r/2}$. Let $h(z) = f(z) $ if $\rho_n(z) \leq a^{(n)}$ and $0$ otherwise $\forall n \in [N]$. We first make the following assumption:

\begin{align}
\begin{aligned}
\mathbb{E}[|\frac{1}{N} \sum_{n \in Q \subseteq [N]}f(X^{(n)})|; \forall n \in Q \subseteq [N],  \rho_n(X^{(n)}) > a^{(n)} ] \leq \\ \mathbb{E}[|\frac{1}{N} \sum_{n \in [N]}f(X^{(n})|; \forall n \in [N],  \rho_n(X^{(n)}) > a^{(n)} ].
\label{eq:assumption}
\end{aligned}
\end{align}

This is indeed a reasonable assumption. To see this, following \citep{chatterjee2018sample}, we note that $\log \rho_n(Z)$ is concentrated around its expected value, which is $L^{(n)} = D_{KL}(q_{\phi^{(n)}} || p_{\theta} )$, in many scenarios. Therefore, for small $t$ (and $t$ is indeed negligibly small in our experiments), the events $\mathbf{1}\{\forall n \in Q \subseteq [N],  \rho_n(X^{(n)}) > a^{(n)}\}$ occur with the approximately same frequency for each set $Q \subseteq [N]$ since the likelihood of event $\mathbf{1}\{\rho_n(X^{(n)}) > a^{(n)}\}$ is close to being uniform.  Consider also that $|\frac{1}{N} \sum_{n \in Q \subseteq [N]}f(X^{(n)})| \leq |\frac{1}{N} \sum_{n \in [N]}f(X^{(n)})|$ holds when $f(X^{n})$'s have the same signs per coordinate for each $n=1, \dots, N$, which is a realistic assumption given that the clients are assumed to be able to train a joint model and hence should not have opposite signs in the updates very often. With these two observations, we argue that the assumption in (\ref{eq:assumption}) is indeed reasonable for many scenarios, including FL. 

Now, going back to the proof, from triangle inequality, we have,

\begin{align}
    |I_K(f) - I(f)| & \leq |I_K(f) - I_K(h)| + |I_K(h) - I(h)| + |I(h) - I(f)|.
\end{align}

First, note that by Cauchy-Schwarz inequality and by the assumption in (\ref{eq:assumption}), we have 

\begin{align}
\begin{split}
     |I(h) -  I(f)|  = &\sum_{Q \subseteq [N]} \mathbb{E}[|\frac{1}{N} \sum_{m \in Q }f(X^{(m)})|; \forall n \in Q,  \rho_n(X^{(n)}) > a^{(n)} ] \cdot \\
    & \cdot \mathbb{P}(\forall n \in Q,  \rho_n(X^{(n)}) > a^{(n)} ) \\
\end{split}
\\
\begin{split}
    & \leq  \mathbb{E}[|\frac{1}{N} \sum_{m \in [N]}f(X^{(m)})|; \forall n \in [N],  \rho_n(X^{(n)}) > a^{(n)} ] \sum_{Q \subseteq [N]}  \mathbb{P}(\forall n \in Q,  \rho_n(X^{(n)}) > a^{(n)} ) \\
\end{split}
\\
\begin{split}
= \mathbb{E}[|\frac{1}{N} \sum_{m \in [N]}f(X^{(m)})|; \forall n \in [N],  \rho_n(X^{(n)}) > a^{(n)} ] \label{eq1} \\
\end{split}
\\
\begin{split}
    = \int_{\mathbf{x}^{(1)}, \dots ,\mathbf{x}^{(N)}} |\frac{1}{N}\sum_{n=1}^N f(\mathbf{x}^{(n)})| \cdot  1\{\forall n \in [N],  \rho_n(\mathbf{x}^{(n)}) > a^{(n)} \} \prod_{n=1}^N dq_{\phi^{(n)}}(\mathbf{x}^{(n)}) 
\end{split}
\\
\begin{split}
\leq & \sqrt{\int_{\mathbf{x}^{(1)}, \dots, \mathbf{x}^{(N)}}  |\frac{1}{N}\sum_{m=1}^N f(\mathbf{x}^{(m)})|^2 \cdot \prod_{n=1}^N dq_{\phi^{(n)}}(\mathbf{x}^{(n)}) } \cdot   \\
& \cdot \sqrt{\int_{\mathbf{x}^{(1)}, \dots, \mathbf{x}^{(N)}}  1\{\forall n \in [N],  \rho_n(\mathbf{x}^{(n)}) > a^{(n)} \} \prod_{n=1}^N dq_{\phi^{(n)}}(\mathbf{x}^{(n)})}  \\
\end{split}
\\
\begin{split}
= \sqrt{\mathbb{E}_{X^{(n)} \sim q_{\phi^{(n)}}, \forall n \in [N]} [(\frac{1}{N}\sum_{m=1}^N f(X^{(m)}))^2] } \cdot \sqrt{\mathbb{P} (\forall n \in [N],  \rho_n(X^{(n)}) > a^{(n)})} \\
\end{split}
\\
\begin{split}
= ||f||_{\mathbf{q_{\phi}}} \cdot \sqrt{\mathbb{P}(\forall n \in [N],  \rho_n(X^{(n)}) > a^{(n)})}. \label{eq2}
\end{split}
\end{align}

Similarly, 

\begin{align}
    \mathbb{E}|I_K(f) - I_K(h)| &= \mathbb{E} \left | \frac{1}{\prod_{n=1}^N K^{(n)}} \sum_{k^{(1)}=1}^{K^{(1)}} \dots \sum_{k^{(N)}=1}^{K^{(N)}} \frac{1}{N}(\sum_{m=1}^N f(Y^{(m)}_{[k^{(m)}]}) - h(Y^{(m)}_{[k^{(m)}]})) \prod_{n=1}^N \rho_n (Y^{(n)}_{[k^{(n)}]}) \right | \\
    & \leq \mathbb{E} \left | \frac{1}{N}(\sum_{m=1}^N f(Y^{(m)}_{[k^{(m)}]}) - h(Y^{(m)}_{[k^{(m)}]})) \prod_{n=1}^N \rho_n(X^{(n)}) \right |  \\
    &= \mathbb{E} [|\frac{1}{N}\sum_{m=1}^N f(X^{(m)})|; \forall n \in [N],  \rho_n(X^{(n)}) > a^{(n)} ] \label{eq3} \\
    & \leq ||f||_{\mathbf{q_{\phi}}} \cdot \sqrt{\mathbb{P}(\forall n \in [N],  \rho_n(X^{(n)}) > a^{(n)})}. \label{eq4}
\end{align}

From (\ref{eq3}) to (\ref{eq4}), we follow the same steps in (\ref{eq1})-(\ref{eq2}).

Finally, note that

\begin{align}
    \mathbb{E} |I_K(h) - I(h)| & \leq \sqrt{Var(I_K(h))} \\
    &= \sqrt{\frac{1}{\prod_{n=1}^N K^{(n)}} Var\left (\frac{1}{N} \sum_{m=1}^N h(Y^{(m)}_{[1]}) \cdot \prod_{n=1}^N \rho_n(Y^{(n)}_{[1]}) \right )} \\
    & \leq \sqrt{\frac{1}{\prod_{n=1}^N K^{(n)}} \mathbb{E} \left [ (\frac{1}{N} \sum_{m=1}^N h(Y^{(n)}_{[1]}))^2 \prod_{n=1}^N (\rho_n(Y^{(n)}_{[1]}))^2\right ]} \\
    & \leq \sqrt{ \frac{\prod_{n=1}^N a^{(n)}}{\prod_{n=1}^N K^{(n)}}  \mathbb{E} \left [ ( \frac{1}{N} \sum_{m=1}^N f(Y^{(m)}_{[1]}))^2  \prod_{n=1}^N \rho_n(Y^{(n)}_{[1]}) \right]} \\
    &= ||f||_{\mathbf{q_{\phi}}} \prod_{n=1}^N \left (\frac{a^{(n)}}{K^{(n)}} \right )^{1/2}.
\end{align}

Combining the upper bounds above, we get

\begin{align}
    \mathbb{E} \left [|I_K(f) - I(f)| \right ] &\leq ||f||_{\mathbf{q_{\phi}}} \left (  \prod_{n=1}^N \left (\frac{a^{(n)}}{K^{(n)}} \right )^{1/2} + 2 \sqrt{\prod_{n=1}^N \mathbb{P} \left ( \log \rho_n(X^{(n)}) > \log a^{(n)} \right)} \right ) \\
    & = ||f||_{\mathbf{q_{\phi}}} \left (e^{-Nr/4} + 2 \sqrt { \prod_{n=1}^N  \mathbb{P} \left  ( \log \rho_n(X^{(n)}) > L^{(n)} + r/2 \right ) }\right ) \\
    & = ||f||_{\mathbf{q_{\phi}}} \left (e^{-Nr/4} + 2 \sqrt{ \prod_{n=1}^N \mathbb{P} \left (\log \rho_n(X^{(n)}) >  D_{KL} (q_{\phi^{(n)}} || p)  + r/2 \right ) } \right ). 
\end{align}
This completes the proof of the first part of the theorem.

For the converse part, suppose $K^{(n)}=e^{L^{(n)}- r}$ and $a^{(n)}=e^{L^{(n)}- r/2}$ $\forall n \in [N]$. Then,

\begin{align}
    & \mathbb{P}(  I_K(\mathbf{1})  \geq 1  - \delta ) = \mathbb{P} \left ( \frac{1}{\prod_{n=1}^N K^{(n)}} \sum_{k_1=1}^{K_1} \dots \sum_{k_N=1}^{K_N} \prod_{n=1}^N \rho_n(Y^{(n)}_{[k^{(n)}]}) \geq 1 - \delta \right ) \label{line1} \\ 
    \begin{split}
    \leq & \mathbb{P} \left ( \max_{1 \leq k \leq K^{(n)} } \rho_n(Y^{(n)}_{[k]}) > a^{(n)}, \forall n \in [N] \right ) \\ 
    & + \mathbb{P} \left ( \frac{1}{\prod_{n=1}^N K^{(n)}}\sum_{k^{(1)}=1}^{K^{(1)}}  \dots \sum_{k^{(N)}=1}^{K^{(N)}} \prod_{n=1}^N \rho_n(Y^{(n)}_{[k^{(n)}]}) 1\{ \forall n \in [N],  \rho_n(Y^{(n)}_{[k^{(n)}]}) \leq a^{(n)}\}\geq 1 - \delta  \right ) \label{line2} \\ 
    \end{split}
    \\
    \begin{split}
         \leq & \sum_{k^{(1)}=1}^{K^{(1)}} \dots \sum_{k^{(N)}=1}^{K^{(N)}} \mathbb{P} \left ( \rho_n(Y^{(n)}_{[k^{(n)}]}) > a^{(n)}, \forall n \in [N] \right ) \\
         & + \frac{1}{1-\delta}\mathbb{E} \left [ \frac{1}{\prod_{n=1}^N K^{(n)}} \sum_{k^{(1)}=1}^{K^{(1)}}  \dots \sum_{k^{(N)}=1}^{K^{(N)}} \prod_{n=1}^N \rho_n(Y^{(n)}_{[k^{(n)}]}) 1\{ \forall n \in [N],  \rho_n(Y^{(n)}_{[k^{(n)}]}) \leq a^{(n)}\} \right ] \\ \label{line3}
    \end{split}
    \\
    & \leq \frac{1}{\prod_{n=1}^N a^{(n)}} \sum_{k^{(1)}=1}^{K^{(1)}} \dots \sum_{k^{(N)}=1}^{K^{(N)}} \prod_{n=1}^N \mathbb{E} \left [  \rho_n(Y^{(n)}_{[k^{(n)}]})\right ] + \frac{1 - \prod_{n=1}^N\mathbb{P} \left (  \rho_n(Z) \geq a^{(n)} \right )}{1- \delta} \\
    & = \prod_{n=1}^N \frac{K^{(n)}}{a^{(n)}} + \frac{\prod_{n=1}^N\mathbb{P} \left (  \rho_n(Z) \leq a^{(n)} \right )}{1- \delta} \\
    & = e^{-  Nr/2} + \frac{\prod_{n=1}^N 
 \mathbb{P} \left ( \log \rho_n(X^{(n)}) \leq  D_{KL}(q_{\phi^{(n)}} || p_{\theta}) - r/2 \right )}{1- \delta},
\end{align}

where from (\ref{line1}) to (\ref{line3}) and  (\ref{line2}) to (\ref{line3}), we use Markov's inequality. This completes the proof of the second inequality in the theorem statement.
\end{proof}

Now, we restate Theorem~\ref{thm2_main} below and provide the proof afterward. 

\begin{theorem}[Theorem~\ref{thm2_main}] \label{thm2} Let all notations be as in Theorem~\ref{thm1} and let $J_K(f)$ be the estimate defined in (\ref{def:JK}). Suppose that $K^{(n)}=\exp{\left ( L^{(n)} + r \right )}$ for some $r \geq 0$. Let

\begin{align}
    \epsilon = \left ( e^{-Nr/4} + 2 \sqrt{  \prod_{n=1}^N \mathbb{P} ( \log \rho_n(X^{(n)}) > L^{(n)} + r/2) }\right )^{1/2}.
\end{align}
Then
\begin{align}
    \mathbb{P} \left(|J_K(f) - I(f)| \geq \frac{2 ||f||_{\mathbf{q_{\phi}}} \epsilon}{1 - \epsilon} \right ) \leq 2 \epsilon.
\end{align}
\end{theorem}

\begin{proof}
    Suppose that $K^{(n)} = e ^{L^{(n)}+ r}$ and $a^{(n)} = e^{L^{(n)}+r/2}$ $\forall n \in [N]$. Let

    \begin{align}
        b = \sqrt{\prod_{n=1}^N \frac{a^{(n)}}{K^{(n)}} }+ 2 \sqrt{\prod_{n=1}^N \mathbb{P} \left( \rho_n(X^{(n)}) > a^{(n)} \right )}.
    \end{align}

Then, by Theorem~\ref{thm1}, for any $\epsilon, \delta \in (0,1)$,

\begin{align}
    \mathbb{P}\left ( |I_K(1) - 1| \geq \epsilon \right ) \leq \frac{b}{\epsilon}
\end{align}
and 
\begin{align}
    \mathbb{P} \left (| I_K(f) - I(f)| \geq \delta \right ) \leq \frac{||f||_{\mathbf{q_{\phi}}}b}{\delta}. 
\end{align}

Now, if $|I_K(f) - I(f)| < \delta$ and $|I_K(1)-1| < \epsilon$, then

\begin{align}
    |J_K(f) - I(f)| &= \left | \frac{I_K(f)}{I_K(1)} - I(f)\right | \\
    & \leq \frac{|I_K(f) - I(f)| + |I(f)| |1-I_K(1)|}{I_K(1)} \\
    & < \frac{\delta + |I(f)| \epsilon}{1 - \epsilon}.
\end{align}

Taking $\epsilon = \sqrt{b}$ and $\delta = ||f||_{\mathbf{q_{\phi}}} \epsilon$ completes the proof of the first inequality in the theorem statement. Note that if $\epsilon$ is bigger than $1$, the bound is true anyway.

This completes the proof of the theorem. 
\end{proof}

\newpage

\section{Additional Details on Example Use Cases of \texttt{KLMS}} \label{app_details_klm_examples}
Here, we present the pseudocodes for the example use cases of $\mathtt{KLMS}$ we covered in the main body. 

\subsection{\texttt{FedPM-KLMS}} \label{app:fedpm_klm}
The pseudocode for \texttt{FedPM-KLMS} can be found in Algorithm~\ref{algorithm:fedpm_klm}. 

\begin{algorithm}[!h]
        {\bf Hyperparameters:} thresholds to update block locations $\bar{D}_{KL}^{\text{max}}$ and $\bar{D}_{KL}^{\text{min}}$, maximum block size \texttt{MAX\_BLOCK\_SIZE}.\\
        {\bf Inputs:} number of iterations $T$, initial block size $S$, number of samples $K$, initial number of blocks $M= \lceil \frac{d}{S} \rceil$, target KL divergence $D_{KL}^{\text{target}}$. \\
        {\bf Output:} random \texttt{SEED} and binary mask parameters $\bm{m}^{(T)}.$\\
    \vspace{-.1in}
        \begin{algorithmic}
            \STATE{At the server, initialize a random network with weight vector $\bm{w}^{\text{init}} \in \mathbb{R}^d$ using a random \texttt{SEED}, and broadcast it to the clients; initialize the random score vector $\bm{s}^{(0, g)} \in \mathbb{R}^d$, and compute $\theta^{(0, g)} \gets \text{Sigmoid}(\bm{s}^{(0, g)})$, Beta priors $\bm{\alpha}^{(0)} = \bm{\beta}^{(0)} = \bm{\lambda}_0$; initialize \texttt{UPDATE}$\gets$TRUE and the block locations $I_i^{(t)}=[(i-1)S: iS]$ for $i=1, \dots, M$ and broadcast to the clients.}
            \FOR{$t=1, \dots, T$} 
            \STATE{Sample a subset $\mathcal{C}_t \subset \{1, \dots, N\}$ of $|\mathcal{C}_t| = C$ clients without replacement.}
            \STATE{\textbf{On Client Nodes:}}
            \FOR{$c \in \mathcal{C}_t$}
            \STATE{Compute $\phi^{(t, c)}$ as in \texttt{FedPM} in Algorithm~\ref{algorithm:fedpm}.}
            \IF{\texttt{UPDATE}}
            \STATE{$\{k^*_{[i]}\}_{i=1}^M, I^{(t, c)} \gets  \texttt{Adaptive-KLMS}(\text{Bern}(\theta^{(t, g)}), \text{Bern}(\phi^{(t, c)}),  I^{(t)}, D_{KL}^{\text{target}})$ // See Algorithm~\ref{algorithm:adaptive_sample_index}.}
            \STATE{$M \gets \text{length}(I^{(t,c)})$. // New number of blocks.}
            \ELSE
            \STATE{$\{k^*_{[i]}\}_{i=1}^M \gets  \texttt{Adaptive-KLMS}(\text{Bern}(\theta^{(t, g)}), \text{Bern}(\phi^{(t, c)}),  I^{(t)}, D_{KL}^{\text{target}})$ // See Algorithm~\ref{algorithm:adaptive_sample_index}.}
            \ENDIF
            \STATE{Send $\{k^*_{[i]}\}_{i=1}^M$ with $K \cdot M$ bits and the average KL divergence across blocks $\bar{D}_{KL}^{(t,c)} \gets \frac{1}{M} \sum_{m=1}^M D_{KL}(\text{Bern}(\phi^{(t,c)}_{[I_m]}) \| \text{Bern}(\theta^{(t,g)}_{[I_m]}))$ with $32$ bits to the server.}
            \IF{\text{UPDATE}}
            \STATE{Send $I^{(t, c)}$ with $M \cdot \log_2($\texttt{MAX\_BLOCK\_SIZE}$)$ bits.}
            \ENDIF
            \ENDFOR \\
            \STATE
            \vspace{-.1in}
            \STATE{\textbf{On the Server Node:}}
            \STATE{Receive the selected indices $\{k^*_{[i]}\}_{i=1}^M$, and the average KL divergences $\{\Bar{D}^{(t, c)}_{KL}\}_{c \in \mathcal{C}_t}$.}
            \STATE{Compute $\Bar{D}_{KL}^{(t)} = \frac{1}{C} \sum_{c \in \mathcal{C}_t} \Bar{D}^{(t, c)}_{KL}$.}
            \IF{\texttt{UPDATE}}
            \STATE{$I^{(t)} \gets \texttt{Aggregate-Block-Locations}\left(\{I^{(t, c)}\}_{c \in \mathcal{C}_t}\right)$ // See Algorithm~\ref{algorithm:aggregate_indices}.}
            \STATE{\texttt{UPDATE} $=$ \texttt{False}.} 
            \ELSE
            \STATE{$I^{(t, c)} \gets I^{(t)}$ for all $c \in \mathcal{C}_t$.}
            \STATE{\textbf{if} $\Bar{D}_{KL}^{(t)} > \bar{D}_{KL}^{\text{max}}$ \textbf{or} $\Bar{D}_{KL}^{(t)} < \bar{D}_{KL}^{\text{min}}$ \textbf{then} \texttt{UPDATE} $=$ \texttt{True} \textbf{else} \texttt{UPDATE} $=$ \texttt{False}.}
            \ENDIF
            \FOR{$c \in \mathcal{C}_t$}
            \STATE{$\{\bm{\hat{m}}^{(t,c)}_{[i]}\}_{i=1}^M \gets$ \texttt{KLMS}-Decoder$(\text{Bern}(\theta^{(t)}), I^{(t, c)}, K)$ // See Algorithm~\ref{algorithm:klms_decoder}.}
            \ENDFOR
            \STATE{$\theta^{(t)} = $ BayesAgg$\left(\{\bm{\hat{m}}^{(t, c)}\}_{c \in \mathcal{C}_t}, t\right)$ // See Algorithm~\ref{algorithm:aggregation}.}
            \STATE{Broadcast \texttt{UPDATE}, $I^{(t)}$ and $\theta^{(t)}$ to the clients.}
            \ENDFOR \\
            \STATE{Sample $\bm{m}^{\text{final}} \sim \text{Bern}(\theta^{(T)})$ and return the final model $\dot{\bm{w}}^{\text{final}} \gets \bm{m}^{\text{final}} \odot \bm{w}^{\text{init}}$.}
            \STATE{}
            \end{algorithmic}
        \caption{\texttt{FedPM-KLMS}.}
        \label{algorithm:fedpm_klm}
    % \vspace{-1mm}
    \end{algorithm}

\subsection{\texttt{QSGD-KLMS}} \label{app:qsgd_klm}
The pseudocode for \texttt{QSGD-KLMS} can be found in Algorithm~\ref{algorithm:qsgd_klm}. 
\begin{algorithm}[!h]
        {\bf Hyperparameters:} server learning rate $\eta_S$,  thresholds to update block locations $\bar{D}_{KL}^{\text{max}}$, $\bar{D}_{KL}^{\text{min}}$, maximum block size \texttt{MAX\_BLOCK\_SIZE}.\\
        {\bf Inputs:} number of iterations $T$, initial block size $S$, number of samples $K$, initial number of blocks $M=\lceil \frac{d}{S} \rceil$, target KL divergence $D_{KL}^{\text{target}}$.\\
        {\bf Output:} Final model $\bm{w}^{(T)}$.\\
    \vspace{-.1in}
        \begin{algorithmic}
            \STATE{At the server, initialize a random network parameters $\bm{w}^{(0)} \in \mathbb{R}^d$ and broadcast it to the clients; initialize \texttt{UPDATE}$\gets$TRUE and the block locations $I_i^{(t)}=[(i-1)S: iS]$ for $i=1, \dots, M$ and broadcast to the clients. }
            \FOR{$t=1, \dots, T$} 
            \STATE{Sample a subset $\mathcal{C}_t \subset \{1, \dots, N\}$ of $|\mathcal{C}_t| = C$ clients without replacement.}
            \STATE{\textbf{On Client Nodes:}}
            \FOR{$c \in \mathcal{C}_t$}
            \STATE{Receive the empirical frequency from the previous round $p_{\theta^{(t)}}$ from the server.}
            \STATE{Compute $\mathbf{v}^{(t,c)}$ as in \texttt{QSGD} in Algorithm~\ref{algorithm:qsgd}.}
            \STATE{Compute the local client-only distribution $q_{\phi^{(t,c)}}$ with $\mathbf{v}^{(t,c)}$ using $p_{\texttt{QSGD}}(\cdot)$ in (\ref{eq:qsgd_s1}).}
            \IF{\texttt{UPDATE}}
            \STATE{$\{k^*_{[i]}\}_{i=1}^M, I^{(t, c)} \gets  \texttt{Adaptive-KLMS}(p_{\theta^{(t)}}, q_{\phi^{(t, c)}},  I^{(t)}, D_{KL}^{\text{target}})$ // See Algorithm~\ref{algorithm:adaptive_sample_index}.}
            \STATE{$M \gets \text{length}(I^{(t,c)})$. // New number of blocks.}
            \ELSE
            \STATE{$\{k^*_{[i]}\}_{i=1}^M \gets  \texttt{Adaptive-KLMS}(p_{\theta^{(t)}}, q_{\phi^{(t, c)}},  I^{(t)}, D_{KL}^{\text{target}})$ // See Algorithm~\ref{algorithm:adaptive_sample_index}.}
            \ENDIF
            \STATE{Send $\{k^*_{[i]}\}_{i=1}^M$ with $K \cdot M$ bits and the average KL divergence across blocks $\bar{D}_{KL}^{(t,c)} \gets \frac{1}{M} \sum_{m=1}^M D_{KL}(q_{\phi^{(t,c)}_{[I_m]}} \| p_{\theta^{(t)}_{[I_m]}}))$ with $32$ bits to the server.}
            \IF{\text{UPDATE}}
            \STATE{Send $I^{(c)}$ with $M \cdot \log_2($\texttt{MAX\_BLOCK\_SIZE}$)$ bits.}
            \ENDIF
            \ENDFOR \\
            \STATE
            \STATE{\textbf{On the Server Node:}}
            \STATE{Receive the selected indices $\{k^*_{[i]}\}_{i=1}^M$, and the average KL divergences $\{\Bar{D}^{(t, c)}_{KL}\}_{c \in \mathcal{C}_t}$.}
            \STATE{Compute $\Bar{D}_{KL}^{(t)} = \frac{1}{C} \sum_{c \in \mathcal{C}_t} \Bar{D}^{(t, c)}_{KL}$.}
            \IF{\texttt{UPDATE}}
            \STATE{$I^{(t, c)} \gets \texttt{Aggregate-Block-Locations}\left(\{I^{(t, c)}\}_{c \in \mathcal{C}_t}\right)$ // See Algorithm~\ref{algorithm:aggregate_indices}.}
            \STATE{\texttt{UPDATE} $=$ \texttt{False}.} 
            \ELSE
            \STATE{$I^{(t, c)} \gets I^{(t)}$ for all $c \in \mathcal{C}_t$.}
            \STATE{\textbf{if} $\Bar{D}_{KL}^{(t)} > \bar{D}_{KL}^{\text{max}}$ \textbf{or} $\Bar{D}_{KL}^{(t)} < \bar{D}_{KL}^{\text{min}}$ \textbf{then} \texttt{UPDATE} $=$ \texttt{True} \textbf{else} \texttt{UPDATE} $=$ \texttt{False}.}
            \ENDIF
            \FOR{$c \in \mathcal{C}_t$}
            \STATE{$\{\mathbf{\hat{v}}^{(t,c)}_{[i]}\}_{i=1}^M \gets$ \texttt{KLMS}-Decoder$(p_{\theta^{(t)}}, I^{(t, c)}, K)$ // See Algorithm~\ref{algorithm:klms_decoder}.}
            \STATE{Construct the empirical frequency $p_{\theta^{(t+1)}}$from $\{\mathbf{\hat{v}}^{(t,c)}_{[i]}\}_{i=1}^M$.}
            \ENDFOR
            \STATE{Compute $\bm{w}^{(t)} = \bm{w}^{(t-1)} - \eta_S \frac{1}{C}\sum_{c \in \mathcal{C}_t}\mathbf{\hat{v}}^{(t, c)}$.}
            \STATE{Broadcast \texttt{UPDATE}, $I^{(t)}$, $\bm{w}^{(t)}$, and $p_{\theta^{(t)}}$ to the clients.}
            \ENDFOR \\
            \end{algorithmic}
        \caption{\texttt{QSGD-KLMS}.}
        \label{algorithm:qsgd_klm}
    % \vspace{-1mm}
    \end{algorithm}

\newpage

\subsection{\texttt{SignSGD-KLM}} \label{app:signsgd_klm}
The pseudocode for \texttt{SignSGD-KLMS} can be found in Algorithm~\ref{algorithm:sign_sgd_klm}. 
\begin{algorithm}[!h]
         {\bf Hyperparameters:} server learning rate $\eta_S$,  thresholds to update block locations $\bar{D}_{KL}^{\text{max}}$, $\bar{D}_{KL}^{\text{min}}$, maximum block size \texttt{MAX\_BLOCK\_SIZE}.\\
        {\bf Inputs:} number of iterations $T$, initial block size $S$, number of samples $K$, initial number of blocks $M=\lceil \frac{d}{S} \rceil$, target KL divergence $D_{KL}^{\text{target}}$.\\
        {\bf Output:} Final model $\bm{w}^{(T)}$.\\
    \vspace{-.1in}
        \begin{algorithmic}
            \STATE{At the server, initialize a random network parameters $\bm{w}^{(0)} \in \mathbb{R}^d$ and broadcast it to the clients; initialize \texttt{UPDATE}$\gets$TRUE and the block locations $I_i^{(t)}=[(i-1)S: iS]$ for $i=1, \dots, M$ and broadcast to the clients. }
            \FOR{$t=1, \dots, T$} 
            \STATE{Sample a subset $\mathcal{C}_t \subset \{1, \dots, N\}$ of $|\mathcal{C}_t| = C$ clients without replacement.}
            \STATE{\textbf{On Client Nodes:}}
            \FOR{$c \in \mathcal{C}_t$}
            \STATE{Compute $\mathbf{v}^{(t,c)}$ as in other standard FL frameworks such as \texttt{QSGD} in Algorithm~\ref{algorithm:qsgd}.}
            \STATE{Compute the local client-only distribution $q_{\phi^{(t,c)}}$ with $\mathbf{v}^{(t,c)}$ using $p_{\texttt{SignSGD}}(\cdot)$ in (\ref{eq:stochastic_sign_sgd}).}
            \STATE{$p_{\theta^{(t)}} \gets \text{Unif}(0.5)$ over $\{-1, 1\}$.}
            \IF{\texttt{UPDATE}}
            \STATE{$\{k^*_{[i]}\}_{i=1}^M, I^{(t, c)} \gets  \texttt{Adaptive-KLMS}(p_{\theta^{(t)}}, q_{\phi^{(t, c)}},  I^{(t)}, D_{KL}^{\text{target}})$ // See Algorithm~\ref{algorithm:adaptive_sample_index}.}
            \STATE{$M \gets \text{length}(I^{(t,c)})$. // New number of blocks.}
            \ELSE
            \STATE{$\{k^*_{[i]}\}_{i=1}^M \gets  \texttt{Adaptive-KLMS}(p_{\theta^{(t)}}, q_{\phi^{(t, c)}},  I^{(t)}, D_{KL}^{\text{target}})$ // See Algorithm~\ref{algorithm:adaptive_sample_index}.}
            \ENDIF
            \STATE{Send $\{k^*_{[i]}\}_{i=1}^M$ with $K \cdot M$ bits and the average KL divergence across blocks $\bar{D}_{KL}^{(t,c)} \gets \frac{1}{M} \sum_{m=1}^M D_{KL}(q_{\phi^{(t,c)}_{[I_m]}} \| p_{\theta^{(t,g)}_{[I_m]}}))$ with $32$ bits to the server.}
            \IF{\text{UPDATE}}
            \STATE{Send $I^{(t, c)}$ with $M \cdot \log_2($\texttt{MAX\_BLOCK\_SIZE}$)$ bits.}
            \ENDIF
            \ENDFOR \\
            \STATE
            \STATE{\textbf{On the Server Node:}}
            \STATE{Receive the selected indices $\{k^*_{[i]}\}_{i=1}^M$, and the average KL divergences $\{\Bar{D}^{(t, c)}_{KL}\}_{c \in \mathcal{C}_t}$.}
            \STATE{Compute $\Bar{D}_{KL}^{(t)} = \frac{1}{C} \sum_{c \in \mathcal{C}_t} \Bar{D}^{(t, c)}_{KL}$.}
            \IF{\texttt{UPDATE}}
            \STATE{$I^{(t)} \gets \texttt{Aggregate-Block-Locations}\left(\{I^{(t, c)}\}_{c \in \mathcal{C}_t}\right)$ // See Algorithm~\ref{algorithm:aggregate_indices}.}
            \STATE{\texttt{UPDATE} $=$ \texttt{False}.} 
            \ELSE
            \STATE{$I^{(t, c)} \gets I^{(t)}$ for all $c \in \mathcal{C}_t$.}
            \STATE{\textbf{if} $\Bar{D}_{KL}^{(t)} > \bar{D}_{KL}^{\text{max}}$ \textbf{or} $\Bar{D}_{KL}^{(t)} < \bar{D}_{KL}^{\text{min}}$ \textbf{then} \texttt{UPDATE} $=$ \texttt{True} \textbf{else} \texttt{UPDATE} $=$ \texttt{False}.}
            \ENDIF
            \FOR{$c \in \mathcal{C}_t$}
            \STATE{$\{\mathbf{\hat{v}}^{(t,c)}_{[i]}\}_{i=1}^M \gets$ \texttt{KLMS}-Decoder$(p_{\theta^{(t)}}, I^{(t, c)}, K)$ // See Algorithm~\ref{algorithm:klms_decoder}.}
            \ENDFOR
            \STATE{Compute $\bm{w}^{(t)} = \bm{w}^{(t-1)} - \eta_S \frac{1}{C}\sum_{c \in \mathcal{C}_t}\mathbf{\hat{v}}^{(t, c)}$.}
            \STATE{Broadcast \texttt{UPDATE}, $I^{(t)}$ and $\bm{w}^{(t)}$ to the clients.}
            \ENDFOR \\
            \end{algorithmic}
        \caption{\texttt{SignSGD-KLMS}.}
        \label{algorithm:sign_sgd_klm}
    % \vspace{-1mm}
    \end{algorithm}

\newpage 

\subsection{\texttt{SGLD-KLMS}} \label{app:sgld_klm}
The pseudocode for \texttt{SGLD-KLMS} can be found in Algorithm~\ref{algorithm:sgld_klm}.

\begin{algorithm}[!h]
        {\bf Hyperparameters:} server learning rate $\eta_S$, minibatch size $B$, thresholds to update block locations $\bar{D}_{KL}^{\text{max}}$, $\bar{D}_{KL}^{\text{min}}$, maximum block size \texttt{MAX\_BLOCK\_SIZE}.\\
        {\bf Inputs:} number of iterations $T$, initial block size $S$, number of samples $K$, initial number of blocks $M=\lceil \frac{d}{S} \rceil$, target KL divergence $D_{KL}^{\text{target}}$.\\
        {\bf Output:} samples $\left\{\theta^{(t)}\right\}_{t=1}^T$.\\
    \vspace{-.1in}
        \begin{algorithmic}
            \STATE{At the server, initialize a random network with weight vector $\theta^{(0)} \in \mathbb{R}^d$ and broadcast it to the clients; initialize \texttt{UPDATE}$\gets$TRUE and the block locations $I_i^{(t)}=[(i-1)S: iS]$ for $i=1, \dots, M$ and broadcast to the clients.}
            \FOR{$t=1, \dots, T$} 
            \STATE{Sample a subset $\mathcal{C}_t \subset \{1, \dots, N\}$ of $|\mathcal{C}_t|=C$ clients without replacement.}
            \STATE{\textbf{On Client Nodes:}}
            \FOR{$c \in \mathcal{C}_t$}
            \STATE{Receive $\theta^{(t-1)}$ from the server and set $\phi^{(t, c)} \gets \theta^{(t-1)}$.}
            \STATE{Compute a stochastic gradient of the potential $H(\phi^{(t, c)})$ as in \texttt{QLSD} in Algorithm~\ref{algorithm:qlsd}.}
            \STATE{Set $p_{\theta^{(t)}} \gets \mathcal{N}\left(0, \sqrt{\frac{2}{\gamma C^2}} \bm{I}_d\right)$.}
            \STATE{Set $q_{\phi^{(t, c)}} \gets \mathcal{N}\left(H(\phi^{(t, c)}), \sqrt{\frac{2}{\gamma C^2}} \bm{I}_d\right)$.}
            \IF{\texttt{UPDATE}}
            \STATE{$\{k^*_{[i]}\}_{i=1}^M, I^{(t, c)} \gets  \texttt{Adaptive-KLMS}(p_{\theta^{(t)}}, q_{\phi^{(t, c)}},  I^{(t)}, D_{KL}^{\text{target}})$ // See Algorithm~\ref{algorithm:adaptive_sample_index}.}
            \STATE{$M \gets \text{length}(I^{(t,c)})$. // New number of blocks.}
            \ELSE
            \STATE{$\{k^*_{[i]}\}_{i=1}^M \gets  \texttt{Adaptive-KLMS}(p_{\theta^{(t)}}, q_{\phi^{(t, c)}},  I^{(t)}, D_{KL}^{\text{target}})$ // See Algorithm~\ref{algorithm:adaptive_sample_index}.}
            \ENDIF
            \STATE{Send $\{k^*_{[i]}\}_{i=1}^M$ with $K \cdot M$ bits and the average KL divergence across blocks $\bar{D}_{KL}^{(t,c)} \gets \frac{1}{M} \sum_{m=1}^M D_{KL}(q_{\phi^{(t,c)}_{[I_m]}} \| p_{\theta^{(t,g)}_{[I_m]}}))$ with $32$ bits to the server.}
            \IF{\text{UPDATE}}
            \STATE{Send $I^{(t, c)}$ with $M \cdot \log_2($\texttt{MAX\_BLOCK\_SIZE}$)$ bits.}
            \ENDIF
            \ENDFOR \\
            \STATE
            \STATE{\textbf{On the Server Node:}}
            \STATE{Receive the selected indices $\{k^*_{[i]}\}_{i=1}^M$, and the average KL divergences $\{\Bar{D}^{(t, c)}_{KL}\}_{c \in \mathcal{C}_t}$.}
            \STATE{Compute $\Bar{D}_{KL}^{(t)} = \frac{1}{C} \sum_{c \in \mathcal{C}_t} \Bar{D}^{(t, c)}_{KL}$.}
            \IF{\texttt{UPDATE}}
            \STATE{$I^{(t)} \gets \texttt{Aggregate-Block-Locations}\left(\{I^{(t, c)}\}_{c \in \mathcal{C}_t}\right)$ // See Algorithm~\ref{algorithm:aggregate_indices}.}
            \STATE{\texttt{UPDATE} $=$ \texttt{False}.} 
            \ELSE
            \STATE{$I^{(t, c)} \gets I^{(t)}$ for all $c \in \mathcal{C}_t$.}
            \STATE{\textbf{if} $\Bar{D}_{KL}^{(t)} > \bar{D}_{KL}^{\text{max}}$ \textbf{or} $\Bar{D}_{KL}^{(t)} < \bar{D}_{KL}^{\text{min}}$ \textbf{then} \texttt{UPDATE} $=$ \texttt{True} \textbf{else} \texttt{UPDATE} $=$ \texttt{False}.}
            \ENDIF
            \FOR{$c \in \mathcal{C}_t$}
            \STATE{$\{\hat{H}(\phi^{(t, c)}_{[i]}\}_{i=1}^M \gets$ \texttt{KLMS}-Decoder$(p_{\theta^{(t)}}, I^{(t, c)}, K)$ // See Algorithm~\ref{algorithm:klms_decoder}.}
            \ENDFOR
            \STATE{Compute $\theta^{(t)} = \theta^{(t-1)} - \eta_S \frac{1}{C}\sum_{c \in \mathcal{C}_t} \hat{H}(\phi^{(t, c)})$.}
            \STATE{Broadcast \texttt{UPDATE}, $I^{(t)}$ and $\theta^{(t)}$ to the clients.}
            \ENDFOR \\
            \end{algorithmic}
        \caption{\texttt{SGLD-KLMS}.}
        \label{algorithm:sgld_klm}
    % \vspace{-1mm}
    \end{algorithm}

\clearpage
\section{Additional Experimental Details}
\label{sec:additional_details_app}
In Tables~\ref{tab:architectures}, ~\ref{tab:resnet18}, and \ref{table:lenet_arch}, we provide the architectures for all the models used in our experiments, namely \texttt{CONV4}, \texttt{CONV6}, ResNet-18, and LeNet. In the non-Bayesian experiments, clients performed three local epochs with a batch size of 128 and a local learning rate of 0.1; while in the Bayesian experiments, they performed one local epoch. We conducted our experiments on NVIDIA Titan X GPUs on an internal cluster server, using 1 GPU per one run.

\setlength{\tabcolsep}{3pt}
\begin{table*}[!h]
%\vspace{-5mm}
\centering
\caption{Architectures for $\mathtt{CONV}$$\mathtt{4}$ and $\mathtt{CONV}$$\mathtt{6}$ models used in the experiments. 
}
\begin{tabular}{ccc}
\toprule
\textbf{Model}&  $\mathtt{CONV}$-$\mathtt{4}$  & $\mathtt{CONV}$-$\mathtt{6}$   \\ \midrule
\centered{ \\ Convolutional \\ Layers}& \centered{ \\ 64, 64, pool \\ 128, 128, pool}
& \centered{ 64, 64, pool \\ 128, 128, pool \\ 256, 256, pool } 
\\ \midrule
\centered{Fully-Connected \\ Layers}& \centered{256, 256, 10}
& \centered{256, 256, 10} 
\\
\bottomrule
\\
\end{tabular}
%}
\label{tab:architectures}
\end{table*}

\begin{table}[h!]
\caption{ResNet-18 architecture.}
\label{tab:resnet18}
\centering
\begin{tabular}{c|c}
\hline
\textbf{Name}& \textbf{Component}\\
\hline
conv1 & $3\times3$ conv, 64 filters. stride 1, BatchNorm \\
\hline
Residual Block 1 & 
$
\begin{bmatrix}
    3 \times 3 \text{ conv, } 64 \text{ filters} \\
    3 \times 3 \text{ conv, } 64 \text{ filters}
\end{bmatrix}
\times 2$ \\
\hline
Residual Block 2 & 
$
\begin{bmatrix}
    3 \times 3 \text{ conv, } 128 \text{ filters} \\
    3 \times 3 \text{ conv, } 128 \text{ filters}
\end{bmatrix}
\times 2$ \\
\hline
Residual Block 3 & $
\begin{bmatrix}
    3 \times 3 \text{ conv, } 256 \text{ filters} \\
    3 \times 3 \text{ conv, } 256 \text{ filters}
\end{bmatrix}
\times 2$ \\
\hline
Residual Block 4 & $
\begin{bmatrix}
    3 \times 3 \text{ conv, } 512 \text{ filters} \\
    3 \times 3 \text{ conv, } 512 \text{ filters}
\end{bmatrix}
\times 2$ \\
\hline
Output Layer & $4 \times 4$ average pool stride 1, fully-connected, softmax \\
\hline
\end{tabular}
\end{table}

\begin{table}[h!]
\centering
\caption{LeNet architecture for MNIST experiments.}
\begin{tabular}{c|c}
\hline
\textbf{Name}& \textbf{Component}\\
\hline
conv1 & [$5 \times 5$ conv, 20 filters, stride 1], ReLU, $2 \times 2$ max pool \\
\hline
conv2 & [$5 \times 5$ conv, 50 filters, stride 1], ReLU, $2 \times 2$ max pool \\
\hline
Linear & Linear $800 \rightarrow 500$, ReLU \\
\hline
Output Layer & Linear $500 \rightarrow 10$ \\
\hline
\end{tabular}
\label{table:lenet_arch}
\end{table}

During non-i.i.d. data split, we choose the size of each client's dataset $|\mathcal{D}^{(n)}| = D_n$ by first uniformly sampling an integer $j_n$ from $\{10, 11, \dots, 100\}$. Then, a coefficient $\frac{j_n}{\sum_n j_j}$ is computed, representing the size of the local dataset $D_n$ as a fraction of the full training dataset size. Moreover, we impose a maximum number of different labels, or classes, $c_{\text{max}}$, that each client can see. This way, highly unbalanced local datasets are generated.

\clearpage
\section{Additional Experimental Results}
\label{sec:additional_results_app}

\subsection{\texttt{KLMS} on a Toy Model} \label{sec:sampling_experiments}

We provide additional insights on \texttt{KLMS} employed in a distributed setup similar to that of FL. Specifically, we design a set of experiments in which the server keeps a global distribution $p = \mathcal{N}(0, 1)$, and $N$ clients need to communicate samples according to their local client-only distributions $\{ q^{(n)}\}_{n=1}^N = \{ \mathcal{N}\left(\mu^{(n)}, 1\right)\}_{n=1}^N$, which are induced by a global and unknown distribution $q = \mathcal{N}\left(\mu, 1\right)$. Each client $n$ applies \texttt{KLMS} (see Figure~\ref{klms_diagram}) to communicate a sample $x^{(n)}$ from $q^{(n)}$ using as coding distribution the global distribution $p$. The server then computes $\hat{\mu} = \frac{1}{N} \sum_{n=1}^N x^{(n)}$ to estimate $\mu$. We study the effect of $N$, i.e., the number of clients communicating their samples, on the estimation of $\mu$ in different scenarios by varying the rate adopted by the clients (Appendix~\ref{sub:estimation_r_n}), and the complexity of the problem (Appendix~\ref{sub:estimation_r_n_het}).

\subsubsection{The Effect of the Overhead $r$}
\label{sub:estimation_r_n}

In this example, we simulate an i.i.d. data split by providing all the clients with the same local client-only distribution $q^{(n)}= \mathcal{N}\left( 0.8, 1 \right)$ $\forall n \in [N]$. We analyze the bias in the estimation of $\mu$ by computing a Monte Carlo average of the discrepancy in (\ref{eq:discrepancy}) (see Figure~\ref{fig:estimation_r}-(\textbf{right})), together with its empirical standard deviation (see Figure~\ref{fig:estimation_r}-(\textbf{left})). From Figure~\ref{fig:estimation_r}, we can observe that, as conjectured, the standard deviation of the gap decreases when $N$ increases, meaning that the estimation is more accurate around its mean value, which is also better for larger values of $N$. Also, as expected, a larger value of the overhead $r$ induces better accuracy.  

\begin{figure}[h!]
     \centering
     \begin{subfigure}[b]{0.48\textwidth}
         \centering
         \includegraphics[width=\textwidth]{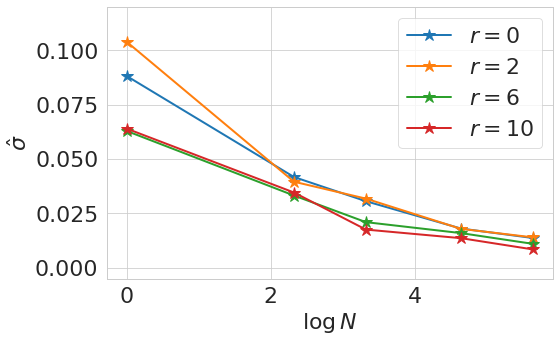}
     \end{subfigure}
     \hfill
     \begin{subfigure}[b]{0.46\textwidth}
         \centering
         \includegraphics[width=\textwidth]{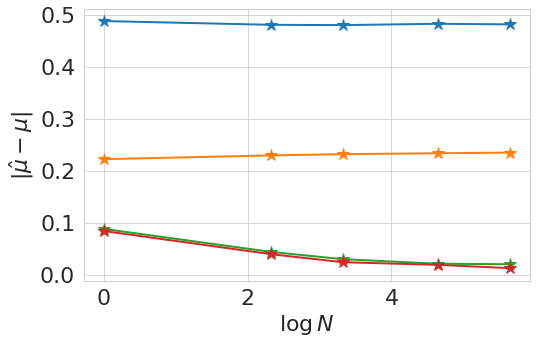}
     \end{subfigure}
        \caption{Estimation gap statistics for different values of $r$, as a function of the number of participating clients $N$. (\textbf{left}) The empirical standard deviation of the estimation gap, computed over $100$ runs. (\textbf{right}) Estimation gap between $\mu$ and $\hat{\mu}$ averaged over $100$ runs.}
        \label{fig:estimation_r}
\end{figure}

\subsubsection{The Effect of Non-i.i.d. Data Split}
\label{sub:estimation_r_n_het}

In this other set of experiments, we simulate a non-i.i.d. data split by inducing, starting from the same global distribution $p$, different local client-only distributions, simulating drifts in updates statistics due to data heterogeneity. Specifically, we set again $\mu = 0.8$, and then, $\forall n \in [N]$, $\mu^{(n)} = 0.8 + u^{(n)}$, where $u^{(n)} \sim \text{Unif}([-\eta, \eta])$, for $\eta \in \{ 0.05, 0.1, 0.25, 0.4\}$. In all experiments, $r = 6$. As we can see from Figure~\ref{fig:estimation_het}, when $N$ is very small ($\sim 1$), then high level of heterogeneity in the update statistics can indeed lead to poor estimation accuracy. However, for reasonable values of $N$, this effect is considerably mitigated, suggesting that for real-world applications of FL, where the number of devices participating to each round can be very large, \texttt{KLMS} can still improve state-of-the-art compression schemes by large margin, as reported in the results of Section~\ref{sec:non_bayesian} and Appendix~\ref{sec:additional_results_app_more_details}.

\begin{figure}[h!]
     \centering
     \begin{subfigure}[b]{0.48\textwidth}
         \centering
         \includegraphics[width=\textwidth]{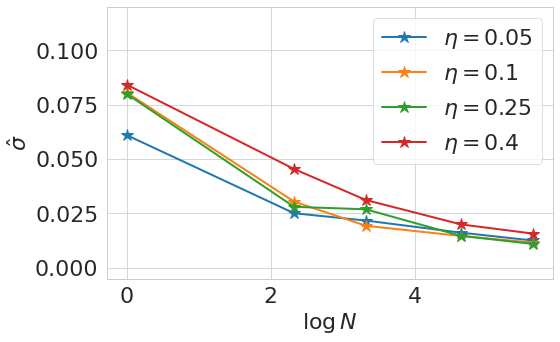}
     \end{subfigure}
     \hfill
     \begin{subfigure}[b]{0.48\textwidth}
         \centering
         \includegraphics[width=\textwidth]{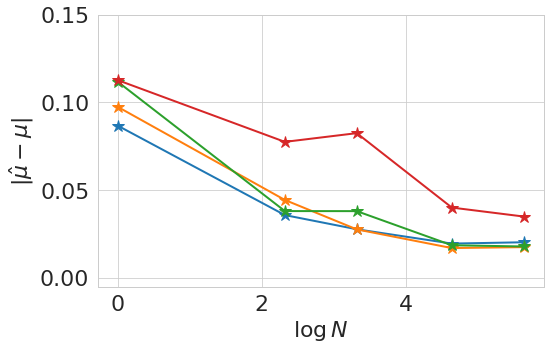}
     \end{subfigure}
        \caption{Estimation gap statistics for different values of $\eta$, as a function of the number of participating clients $N$. (\textbf{left}) The empirical standard deviation of the estimation gap, computed over $100$ runs. (\textbf{right}) Estimation gap between $\mu$ and $\hat{\mu}$ averaged over $100$ runs.}
        \label{fig:estimation_het}
\end{figure}

\subsection{Additional Results with Non-i.i.d. CIFAR-10}
\label{sec:additional_results_app_more_details}
In Figure~\ref{fig:acc_bitrate_noniid_nonbayesian_c2_appendix}, we give the results on \texttt{CONV6} and \texttt{ResNet-18} on non-i.i.d. CIFAR-10 with $c_{\text{max}}=2$ and CIFAR-100 with $c_{\text{max}}=20$, respectively. In both experiments, $ 10$ clients out of $100$ clients participate in each round. It is seen that similar accuracy and bitrate improvements are observed to the non-i.i.d. results in Table~\ref{tab:acc_bitrate_noniid_nonbayesian}.
\begin{figure}[h!]
    \centering
    \includegraphics[width=\textwidth]{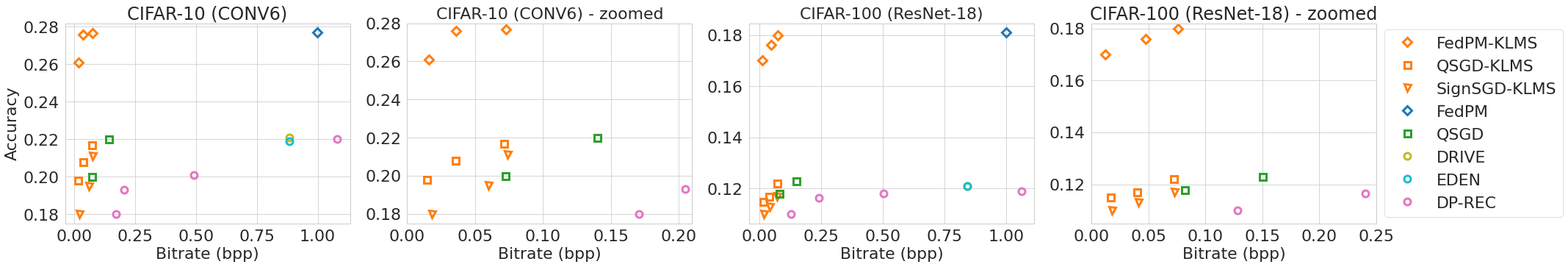}
    %\vspace{-2ex}
    \caption{Comparison of \texttt{FedPM-KLM}, \texttt{QSGD-KLM}, and \texttt{SignSGD-KLM} with \texttt{FedPM}~\citep{isik2023sparse}, \texttt{QSGD}~\citep{alistarh2017qsgd}, \texttt{DRIVE}~\citep{vargaftik2021drive}, \texttt{EDEN}~\citep{vargaftik2022eden}, and \texttt{DP-REC}~\citep{liu2021fedprune} with non i.i.d. split and $10$ out of $100$ clients participating every round.
    }
\vspace{-0.1in}
    \label{fig:acc_bitrate_noniid_nonbayesian_c2_appendix}
\end{figure}

\subsection{Stack Overflow Experiments}
Table~\ref{tab:stackoverflow} shows additional results on the Stack Overflow dataset~\citep{caldas2018leaf} constructed with real posts, where \texttt{KLMS} reduces the bitrate by $16$ times. 

 \begin{table}[h!]
    \centering
\caption{Results on Stack Overflow (real data split).}
\begin{tabular}{c|c|c|cc|cc|cc}
 \toprule
         & \texttt{DRIVE} & \texttt{EDEN} & \texttt{SignSGD} & \texttt{SignSGD-KLMS} & \texttt{FedPM} & \texttt{FedPM-KLMS} & \texttt{QSGD} & \texttt{QSGD-KLMS} \\
 \midrule
  Accuracy  & $0.216$ & $0.216$ & $0.186$ & $\bm{0.211}$&  $0.240$ & $\bm{0.240}$ & $0.210$ & $\bm{0.224}$\\ 
  Bitrate   & $0.980$ & $0.980$ & $1.000$ & $\bm{0.016}$& $0.910$ & $\bm{0.015}$ & $0.120$ & $\bm{0.016}$\\
\bottomrule
\end{tabular}
    \label{tab:stackoverflow}
\end{table}

\subsection{Bayesian FL Experiments with Non-i.i.d. Data and Partial Client Participation}
Table~\ref{tab:bayesian_fl_noniid} shows additional Bayesian FL results with non-iid data split and partial client participation. We cover a variety of combinations of data splits and partial/full participation and observe similar gains as Table~\ref{tab:acc_bitrate_noniid_nonbayesian}.

\begin{table}[h!]
\centering
\caption{Bayesian FL results ($c_{\text{max}}=4$ when non-iid.)}

\begin{tabular}{ccc|cc|cc}
 \toprule
& \multicolumn{2}{c}{$\mathbf{\rho=10/10}$\textbf{, non-iid}} &\multicolumn{2}{c}{$\mathbf{\rho=10/50}$\textbf{, iid}} &\multicolumn{2}{c}{$\mathbf{\rho=10/50}$\textbf{, non-iid}}\\
   & \texttt{QLSD} & \texttt{QLSD-KLMS} & \texttt{QLSD} & \texttt{QLSD-KLMS} & \texttt{QLSD} & \texttt{QLSD-KLMS} \\
 \midrule
  Accuracy  & $0.875$ & $\bm{0.922}$ & $0.868$ & $\textbf{0.922}$ & $0.854$ & $\bm{0.920}$\\ 
  Bitrate   & $0.48$ & $\bm{0.08}$ & $0.50$ & $\bm{0.07}$ & $0.51$ & $\bm{0.06}$\\ 
\bottomrule
\end{tabular} \label{tab:bayesian_fl_noniid}
\vspace{-4mm} 
\end{table}

\subsection{Confidence Intervals}
Finally, we report the confidence intervals for all the experimental results in the paper in Tables~\ref{tab:iid_cifar10_confidence}, \ref{tab:iid_cifar100_confidence}, \ref{tab:iid_mnist_confidence}, \ref{tab:iid_emnist_confidence}, \ref{tab:non_iid_2_class_cifar10_confidence}, \ref{tab:non_iid_4_class_cifar10_confidence}, \ref{tab:non_iid_20_class_cifar100_confidence}, and \ref{tab:non_iid_40_class_cifar100_confidence} corresponding to Tables~\ref{tab:acc_bitrate_iid_nonbayesian} and~\ref{tab:acc_bitrate_noniid_nonbayesian}, and Figure~\ref{fig:acc_bitrate_noniid_nonbayesian_c2_appendix}.

\begin{table}[h!]
\centering
\caption{Average bitrate $\pm \sigma$ vs final accuracy $\pm \sigma$ in i.i.d. split CIFAR-10 with full client participation. The training duration was set to $t_{\text{max}} = 400$ rounds.}
%\resizebox{\columnwidth}{!}{
\begin{tabular}{ ccc }
 \toprule
\textbf{Framework} &  \textbf{Bitrate} & \textbf{Accuracy} \\
 \midrule
 \midrule
\texttt{FedPM-KLMS} (ours) & 0.070 $\pm$ 0.0001 & 0.787 $\pm$ 0.0012\\
\texttt{FedPM-KLMS} (ours) & 0.004 $\pm$ 0.0001 & 0.786 $\pm$ 0.0010\\
\texttt{FedPM-KLMS} (ours) & 0.014 $\pm$ 0.0001 & 0.786 $\pm$ 0.0012 \\
\texttt{QSGD-KLMS} (ours) & 0.071 $\pm$ 0.0001 & 0.765 $\pm$ 0.0011 \\
\texttt{QSGD-KLMS} (ours) & 0.0355 $\pm$ 0.0001 & 0.761 $\pm$ 0.0012 \\
\texttt{QSGD-KLMS} (ours) & 0.0142 $\pm$ 0.0001 & 0.755 $\pm$ 0.0010 \\
\texttt{SignSGD-KLMS} (ours) & 0.072 $\pm$ 0.0002 & 0.745 $\pm$ 0.0008 \\
\texttt{SignSGD-KLMS} (ours) & 0.040 $\pm$ 0.0002 & 0.745 $\pm$ 0.0008 \\
\texttt{SignSGD-KLMS} (ours) & 0.015 $\pm$ 0.0001 & 0.739 $\pm$ 0.0008 \\
\midrule
\texttt{FedPM}~\citep{isik2023sparse} & 0.845 $\pm$ 0.0001 & 0.787 $\pm$ 0.0011 \\
\texttt{QSGD}~\citep{alistarh2017qsgd} & 0.140 $\pm$ 0.0000 & 0.766 $\pm$ 0.0012 \\
\texttt{QSGD}~\citep{alistarh2017qsgd} & 0.072 $\pm$ 0.0000 & 0.753 $\pm$ 0.0013 \\
\texttt{SignSGD}~\citep{bernstein2018signsgd} & 0.993 $\pm$ 0.0012 & 0.705 $\pm$ 0.0021 \\
\texttt{TernGrad}~\citep{wen2017terngrad} & 1.100 $\pm$ 0.0001 & 0.680 $\pm$ 0.0016 \\
\texttt{DRIVE}~\citep{vargaftik2021drive} & 0.890 $\pm$ 0.0000 & 0.760 $\pm$ 0.0010 \\
\texttt{EDEN}~\citep{vargaftik2022eden} & 0.890 $\pm$ 0.0000 & 0.760 $\pm$ 0.0010 \\
\texttt{FedMask}~\citep{li2021fedmask} & 1.000 $\pm$ 0.0001 & 0.620 $\pm$ 0.0017  \\
\texttt{DP-REC}~\citep{triastcyn2021dp} & 1.12 $\pm$ 0.0001 & 0.720 $\pm$ 0.0011 \\
\texttt{DP-REC}~\citep{triastcyn2021dp} & 0.451 $\pm$ 0.0001 & 0.690 $\pm$ 0.0012 \\
\texttt{DP-REC}~\citep{triastcyn2021dp} & 0.188 $\pm$ 0.0001 & 0.640 $\pm$ 0.0011 \\
\texttt{DP-REC}~\citep{triastcyn2021dp} & 0.124 $\pm$ 0.0001 & 0.622 $\pm$ 0.0013 \\
\bottomrule
\end{tabular}%}
\label{tab:iid_cifar10_confidence}
\end{table}

\begin{table}[h!]
\centering
\caption{Average bitrate $\pm \sigma$ vs final accuracy $\pm \sigma$ in i.i.d. split CIFAR-100 with full client participation. The training duration was set to $t_{\text{max}} = 400$ rounds.}
%\resizebox{\columnwidth}{!}{
\begin{tabular}{ ccc }
 \toprule
 \textbf{Framework} &  \textbf{Bitrate} & \textbf{Accuracy} \\
 \midrule
 \midrule
\texttt{FedPM-KLMS} (ours) & 0.072 $\pm$ 0.0001 & 0.469 $\pm$ 0.0010 \\
\texttt{FedPM-KLMS} (ours) & 0.040 $\pm$ 0.0001 & 0.461 $\pm$ 0.0011 \\
\texttt{FedPM-KLMS} (ours) & 0.018 $\pm$ 0.0001 & 0.455 $\pm$ 0.0010 \\
\texttt{QSGD-KLMS} (ours) & 0.074 $\pm$ 0.0001 & 0.327 $\pm$ 0.0010 \\
\texttt{QSGD-KLMS} (ours) & 0.043 $\pm$ 0.0001 & 0.319 $\pm$ 0.0012 \\
\texttt{QSGD-KLMS} (ours) & 0.020 $\pm$ 0.0001 & 0.320 $\pm$ 0.0010 \\
\texttt{SignSGD-KLMS} (ours) & 0.073 $\pm$ 0.0001 & 0.260 $\pm$ 0.0014 \\
\texttt{SignSGD-KLMS} (ours) & 0.041 $\pm$ 0.0001 & 0.259 $\pm$ 0.0014 \\
\texttt{SignSGD-KLMS} (ours) & 0.018 $\pm$ 0.0001 & 0.250 $\pm$ 0.0014 \\
\midrule
\texttt{FedPM}~\citep{isik2023sparse} & 0.880 $\pm$ 0.0001 & 0.470 $\pm$ 0.0010 \\
\texttt{QSGD}~\citep{alistarh2017qsgd} & 0.150 $\pm$ 0.0000 & 0.335 $\pm$ 0.0011 \\
\texttt{QSGD}~\citep{alistarh2017qsgd} & 0.082 $\pm$ 0.0000 & 0.330 $\pm$ 0.0011 \\
\texttt{SignSGD}~\citep{bernstein2018signsgd} & 0.999 $\pm$ 0.0002 & 0.230 $\pm$ 0.0019 \\
\texttt{TernGrad}~\citep{wen2017terngrad} & 1.070 $\pm$ 0.0001 & 0.220 $\pm$ 0.0015 \\
\texttt{DRIVE}~\citep{vargaftik2021drive} & 0.540 $\pm$ 0.0000 & 0.320 $\pm$ 0.0011 \\
\texttt{EDEN}~\citep{vargaftik2022eden} & 0.540 $\pm$ 0.0000 & 0.320 $\pm$ 0.0010 \\
\texttt{FedMask}~\citep{li2021fedmask} & 1.000 $\pm$ 0.0001 & 0.180 $\pm$ 0.0014 \\
\texttt{DP-REC}~\citep{triastcyn2021dp} & 1.06 $\pm$ 0.0001 & 0.280 $\pm$ 0.0012 \\
\texttt{DP-REC}~\citep{triastcyn2021dp} & 0.503 $\pm$ 0.0001 & 0.240 $\pm$ 0.0012 \\
\texttt{DP-REC}~\citep{triastcyn2021dp} & 0.240 $\pm$ 0.0001 & 0.220 $\pm$ 0.0012 \\
\texttt{DP-REC}~\citep{triastcyn2021dp} & 0.128 $\pm$ 0.0001 & 0.170 $\pm$ 0.0012 \\
\bottomrule
\end{tabular}%}
\label{tab:iid_cifar100_confidence}
\end{table}

\begin{table}[h!]
\centering
\caption{Average bitrate $\pm \sigma$ vs final accuracy $\pm \sigma$ in i.i.d. split MNIST with full client participation. The training duration was set to $t_{\text{max}} = 200$ rounds.}
%\resizebox{\columnwidth}{!}{
\begin{tabular}{ ccc }
 \toprule
\textbf{Framework} &  \textbf{Bitrate} & \textbf{Accuracy} \\
 \midrule
 \midrule
\texttt{FedPM-KLMS} (ours) & 0.067 $\pm$ 0.0001 & 0.9945 $\pm$ 0.0001 \\
\texttt{FedPM-KLMS} (ours) & 0.041 $\pm$ 0.0001 & 0.9945 $\pm$ 0.0001 \\
\texttt{FedPM-KLMS} (ours) & 0.014 $\pm$ 0.0001 & 0.9943 $\pm$ 0.0001 \\
\texttt{QSGD-KLMS} (ours) & 0.071 $\pm$ 0.0001 & 0.9940 $\pm$ 0.0001  \\
\texttt{QSGD-KLMS} (ours) & 0.041 $\pm$ 0.0001 & 0.9938 $\pm$ 0.0001  \\
\texttt{QSGD-KLMS} (ours) & 0.019 $\pm$ 0.0001 & 0.9935 $\pm$ 0.0001  \\
\texttt{SignSGD-KLMS} (ours) & 0.0720 $\pm$ 0.0001 & 0.9932 $\pm$ 0.0002 \\
\texttt{SignSGD-KLMS} (ours) & 0.0415 $\pm$ 0.0001 & 0.9930 $\pm$ 0.0002 \\
\texttt{SignSGD-KLMS} (ours) & 0.0230 $\pm$ 0.0001 & 0.9918 $\pm$ 0.0001 \\
\midrule
\texttt{FedPM}~\citep{isik2023sparse} & 0.99 $\pm$ 0.0001 & 0.995 $\pm$ 0.0001 \\
\texttt{QSGD}~\citep{alistarh2017qsgd} & 0.13 $\pm$ 0.0000 & 0.994 $\pm$ 0.0001 \\
\texttt{QSGD}~\citep{alistarh2017qsgd} & 0.080 $\pm$ 0.0000 & 0.994 $\pm$ 0.0001 \\
\texttt{SignSGD}~\citep{bernstein2018signsgd} & 0.999 $\pm$ 0.0012 & 0.990 $\pm$ 0.0004 \\
\texttt{TernGrad}~\citep{wen2017terngrad} & 1.05 $\pm$ 0.0001 & 0.980 $\pm$ 0.0003 \\
\texttt{DRIVE}~\citep{vargaftik2021drive} & 0.91 $\pm$ 0.0000 & 0.994 $\pm$ 0.0001 \\
\texttt{EDEN}~\citep{vargaftik2022eden} &  0.91 $\pm$ 0.0000 & 0.994 $\pm$ 0.0001 \\
\texttt{FedMask}~\citep{li2021fedmask} & 1.0 $\pm$ 0.0001 & 0.991 $\pm$ 0.0003 \\
\texttt{DP-REC}~\citep{triastcyn2021dp} & 0.996 $\pm$ 0.0001 & 0.991 $\pm$ 0.0001  \\
\texttt{DP-REC}~\citep{triastcyn2021dp} & 0.542 $\pm$ 0.0001 & 0.989 $\pm$ 0.0001  \\
\texttt{DP-REC}~\citep{triastcyn2021dp} & 0.191 $\pm$ 0.0001 & 0.988 $\pm$ 0.0001  \\
\texttt{DP-REC}~\citep{triastcyn2021dp} & 0.125 $\pm$ 0.0001 & 0.985 $\pm$ 0.0001  \\
\bottomrule
\end{tabular}%}
\label{tab:iid_mnist_confidence}
\end{table}

\begin{table}[h!]
\centering
\caption{Average bitrate $\pm \sigma$ vs final accuracy $\pm \sigma$ in i.i.d. split EMNIST with full client participation. The training duration was set to $t_{\text{max}} = 200$ rounds.}
%\resizebox{\columnwidth}{!}{
\begin{tabular}{ ccc }
 \toprule
\textbf{Framework} &  \textbf{Bitrate} & \textbf{Accuracy} \\
 \midrule
 \midrule
\texttt{FedPM-KLMS} (ours) & 0.068  $\pm$ 0.0001 & 0.889 $\pm$ 0.0001 \\
\texttt{FedPM-KLMS} (ours) & 0.034 $\pm$ 0.0001 & 0.888 $\pm$ 0.0001 \\
\texttt{FedPM-KLMS} (ours) & 0.017 $\pm$ 0.0001 & 0.885 $\pm$ 0.0001 \\
\texttt{QSGD-KLMS} (ours) & 0.072 $\pm$ 0.0001 & 0.884 $\pm$ 0.0001 \\
\texttt{QSGD-KLMS} (ours) & 0.042 $\pm$ 0.0001 & 0.884 $\pm$ 0.0001\\
\texttt{QSGD-KLMS} (ours) & 0.022 $\pm$ 0.0001 & 0.883 $\pm$ 0.0001 \\
\texttt{SignSGD-KLMS} (ours) & 0.072 $\pm$ 0.0001 & 0.881 $\pm$ 0.0003 \\
\texttt{SignSGD-KLMS} (ours) & 0.044 $\pm$ 0.0001 & 0.880 $\pm$ 0.0003 \\
\texttt{SignSGD-KLMS} (ours) & 0.025 $\pm$ 0.0001 & 0.875 $\pm$ 0.0003 \\
\midrule
\texttt{FedPM}~\citep{isik2023sparse} & 0.890 $\pm$ 0.0001 & 0.890 $\pm$ 0.0001 \\
\texttt{QSGD}~\citep{alistarh2017qsgd} & 0.150 $\pm$ 0.0000 & 0.884 $\pm$ 0.0001 \\
\texttt{QSGD}~\citep{alistarh2017qsgd} & 0.086 $\pm$ 0.0000 & 0.882 $\pm$ 0.0001 \\
\texttt{SignSGD}~\citep{bernstein2018signsgd} & 1.0 $\pm$ 0.0001 & 0.873 $\pm$ 0.0005 \\
\texttt{TernGrad}~\citep{wen2017terngrad} & 1.1 $\pm$ 0.0001 & 0.870 $\pm$ 0.0005 \\
\texttt{DRIVE}~\citep{vargaftik2021drive} & 0.9 $\pm$ 0.0001 & 0.8835 $\pm$ 0.0001 \\
\texttt{EDEN}~\citep{vargaftik2022eden} & 0.9 $\pm$ 0.0001 & 0.8835 $\pm$ 0.0001 \\
\texttt{FedMask}~\citep{li2021fedmask} & 1.0 $\pm$ 0.0001 & 0.862 $\pm$ 0.0005 \\
\texttt{DP-REC}~\citep{triastcyn2021dp} & 1.100 $\pm$ 0.0001 & 0.885 $\pm$ 0.0001 \\
\texttt{DP-REC}~\citep{triastcyn2021dp} & 0.488 $\pm$ 0.0001 & 0.880 $\pm$ 0.0001 \\
\texttt{DP-REC}~\citep{triastcyn2021dp} & 0.196 $\pm$ 0.0001 & 0.873 $\pm$ 0.0001 \\
\texttt{DP-REC}~\citep{triastcyn2021dp} & 0.119 $\pm$ 0.0001 & 0.861 $\pm$ 0.0001 \\
\bottomrule
\end{tabular}%}
\label{tab:iid_emnist_confidence}
\end{table}

\begin{table}[h!]
\centering
\caption{Average bitrate $\pm \sigma$ vs final accuracy $\pm \sigma$ in non-IID split CIFAR-10 with $c_{\text{max}} = 2$, and partial participation with $10$ out of $100$ clients participating every round. The training duration was set to $t_{\text{max}} = 200$ rounds.}
%\resizebox{\columnwidth}{!}{
\begin{tabular}{ ccc }
 \toprule
\textbf{Framework} &  \textbf{Bitrate} & \textbf{Accuracy} \\
 \midrule
 \midrule
\texttt{FedPM-KLMS} (ours) & 0.073 $\pm$ 0.0001 &  0.277 $\pm$ 0.0005 \\
\texttt{FedPM-KLMS} (ours) & 0.036 $\pm$ 0.0001 &  0.276 $\pm$ 0.0005 \\
\texttt{FedPM-KLMS} (ours) & 0.0161 $\pm$ 0.0001 &  0.261 $\pm$ 0.0004 \\
\texttt{QSGD-KLMS} (ours) & 0.071 $\pm$ 0.0001 & 0.277 $\pm$ 0.0005 \\
\texttt{QSGD-KLMS} (ours) & 0.036 $\pm$ 0.0001 & 0.208 $\pm$ 0.0005 \\
\texttt{QSGD-KLMS} (ours) & 0.014 $\pm$ 0.0001 & 0.198 $\pm$ 0.0005 \\
\texttt{SignSGD-KLMS} (ours) & 0.074 $\pm$ 0.0001 & 0.211 $\pm$ 0.0009 \\
\texttt{SignSGD-KLMS} (ours) & 0.060 $\pm$ 0.0001 & 0.195 $\pm$ 0.0008 \\
\texttt{SignSGD-KLMS} (ours) & 0.018 $\pm$ 0.0001 & 0.180 $\pm$ 0.0009 \\
\midrule
\texttt{FedPM}~\citep{isik2023sparse} & 0.997 $\pm$ 0.0001 & 0.277 $\pm$ 0.0006 \\
\texttt{QSGD}~\citep{alistarh2017qsgd} & 0.140 $\pm$ 0.0000 & 0.220 $\pm$ 0.0005 \\
\texttt{QSGD}~\citep{alistarh2017qsgd} & 0.072 $\pm$ 0.0000 & 0.200 $\pm$ 0.0005 \\
\texttt{DRIVE}~\citep{vargaftik2021drive} & 0.885 $\pm$ 0.0000 & 0.221 $\pm$ 0.0005 \\
\texttt{EDEN}~\citep{vargaftik2022eden} & 0.885 $\pm$ 0.0000 & 0.219 $\pm$ 0.0004 \\
\texttt{DP-REC}~\citep{triastcyn2021dp} & 1.080  $\pm$ 0.0001 & 0.220 $\pm$ 0.0007 \\
\texttt{DP-REC}~\citep{triastcyn2021dp} & 0.490  $\pm$ 0.0001 & 0.201 $\pm$ 0.0006 \\
\texttt{DP-REC}~\citep{triastcyn2021dp} & 0.205  $\pm$ 0.0001 & 0.193 $\pm$ 0.0006 \\
\texttt{DP-REC}~\citep{triastcyn2021dp} & 0.171 $\pm$ 0.0001 & 0.180 $\pm$ 0.0006 \\
\bottomrule
\end{tabular}%}
\label{tab:non_iid_2_class_cifar10_confidence}
\end{table}

\begin{table}[h!]
\centering
\caption{Average bitrate $\pm \sigma$ vs final accuracy $\pm \sigma$ in non-IID split CIFAR-10 with $c_{\text{max}} = 4$, and partial participation with $20$ out of $100$ clients participating every round. The training duration was set to $t_{\text{max}} = 200$ rounds.}
%\resizebox{\columnwidth}{!}{
\begin{tabular}{ ccc }
 \toprule
\textbf{Framework} &  \textbf{Bitrate} & \textbf{Accuracy} \\
 \midrule
 \midrule
\texttt{FedPM-KLMS} (ours) & 0.073 $\pm$ 0.0001 & 0.612 $\pm$ 0.0010 \\
\texttt{FedPM-KLMS} (ours) & 0.036 $\pm$ 0.0001 & 0.606 $\pm$ 0.0010 \\
\texttt{FedPM-KLMS} (ours) & 0.016 $\pm$ 0.0001 & 0.599 $\pm$ 0.0010 \\
\texttt{QSGD-KLMS} (ours) & 0.071 $\pm$ 0.0001 & 0.552 $\pm$ 0.0010 \\
\texttt{QSGD-KLMS} (ours) & 0.036 $\pm$ 0.0001 & 0.549 $\pm$ 0.0011 \\
\texttt{QSGD-KLMS} (ours) & 0.014 $\pm$ 0.0001 & 0.545 $\pm$ 0.0010 \\
\texttt{SignSGD-KLMS} (ours) & 0.074 $\pm$ 0.0001 & 0.530 $\pm$ 0.0013 \\
\texttt{SignSGD-KLMS} (ours) & 0.060 $\pm$ 0.0001 & 0.522 $\pm$ 0.0013 \\
\texttt{SignSGD-KLMS} (ours) & 0.018 $\pm$ 0.0001 & 0.518 $\pm$ 0.0013 \\
\midrule
\texttt{FedPM}~\citep{isik2023sparse} & 0.993 $\pm$ 0.0001 & 0.612 $\pm$ 0.0009 \\
\texttt{QSGD}~\citep{alistarh2017qsgd} & 0.140 $\pm$ 0.0000 & 0.552 $\pm$ 0.0010 \\
\texttt{QSGD}~\citep{alistarh2017qsgd} & 0.072 $\pm$ 0.0000 & 0.531 $\pm$ 0.0010 \\
\texttt{DRIVE}~\citep{vargaftik2021drive} & 0.888 $\pm$ 0.0000 & 0.526 $\pm$ 0.0010 \\
\texttt{EDEN}~\citep{vargaftik2022eden} & 0.888 $\pm$ 0.0000 & 0.528 $\pm$ 0.0010 \\
\texttt{DP-REC}~\citep{triastcyn2021dp} &  1.080 $\pm$ 0.0001 & 0.530 $\pm$ 0.0012 \\
\texttt{DP-REC}~\citep{triastcyn2021dp} &  0.490 $\pm$ 0.0001 & 0.521 $\pm$ 0.0012 \\
\texttt{DP-REC}~\citep{triastcyn2021dp} &  0.205 $\pm$ 0.0001 & 0.519 $\pm$ 0.0012 \\
\texttt{DP-REC}~\citep{triastcyn2021dp} &  0.171 $\pm$ 0.0001 & 0.506 $\pm$ 0.0012 \\
\bottomrule
\end{tabular}%}
\label{tab:non_iid_4_class_cifar10_confidence}
\end{table}

\begin{table}[h!]
\centering
\caption{Average bitrate $\pm \sigma$ vs final accuracy $\pm \sigma$ in non-IID split CIFAR-100 with $c_{\text{max}} = 20$, and partial participation with $10$ out of $100$ clients participating every round. The training duration was set to $t_{\text{max}} = 200$ rounds.}
%\resizebox{\columnwidth}{!}{
\begin{tabular}{ ccc }
 \toprule
\textbf{Framework} &  \textbf{Bitrate} & \textbf{Accuracy} \\
 \midrule
 \midrule
\texttt{FedPM-KLMS} (ours) & 0.076 $\pm$ 0.0001 & 0.180 $\pm$ 0.0012 \\
\texttt{FedPM-KLMS} (ours) & 0.048 $\pm$ 0.00101 & 0.176 $\pm$ 0.0011 \\
\texttt{FedPM-KLMS} (ours) & 0.012 $\pm$ 0.0001 & 0.170 $\pm$ 0.0011 \\
\texttt{QSGD-KLMS} (ours) & 0.072 $\pm$ 0.0001 & 0.122 $\pm$ 0.0012 \\
\texttt{QSGD-KLMS} (ours) & 0.040 $\pm$ 0.0001 & 0.117 $\pm$ 0.0012 \\
\texttt{QSGD-KLMS} (ours) & 0.017 $\pm$ 0.0001 & 0.115 $\pm$ 0.0012 \\
\texttt{SignSGD-KLMS} (ours) & 0.073 $\pm$ 0.0001 & 0.117 $\pm$ 0.0014 \\
\texttt{SignSGD-KLMS} (ours) & 0.041 $\pm$ 0.0001 & 0.113 $\pm$ 0.0014 \\
\texttt{SignSGD-KLMS} (ours) & 0.018 $\pm$ 0.0001 & 0.110 $\pm$ 0.0013 \\
\midrule
\texttt{FedPM}~\citep{isik2023sparse} & 0.999 $\pm$ 0.0001 & 0.181 $\pm$ 0.0011 \\
\texttt{QSGD}~\citep{alistarh2017qsgd} &  0.150 $\pm$ 0.0000 & 0.123 $\pm$ 0.0012 \\
\texttt{QSGD}~\citep{alistarh2017qsgd} &  0.082 $\pm$ 0.0000 & 0.118 $\pm$ 0.0012 \\
\texttt{DRIVE}~\citep{vargaftik2021drive} & 0.840 $\pm$ 0.0000 & 0.121 $\pm$ 0.0012 \\
\texttt{EDEN}~\citep{vargaftik2022eden} & 0.840 $\pm$ 0.0000  & 0.121 $\pm$ 0.0012 \\
\texttt{DP-REC}~\citep{triastcyn2021dp} & 1.060 $\pm$ 0.0001 & 0.119  $\pm$ 0.0012\\
\texttt{DP-REC}~\citep{triastcyn2021dp} & 0.503 $\pm$ 0.0001 & 0.118 $\pm$ 0.0013 \\
\texttt{DP-REC}~\citep{triastcyn2021dp} & 0.240  $\pm$ 0.0001 & 0.117 $\pm$ 0.0013\\
\texttt{DP-REC}~\citep{triastcyn2021dp} & 0.128 $\pm$ 0.0001 & 0.110  $\pm$ 0.0013 \\
\bottomrule
\end{tabular}%}
\label{tab:non_iid_20_class_cifar100_confidence}
\end{table}

\begin{table}[h!]
\centering
\caption{Average bitrate $\pm \sigma$ vs final accuracy $\pm \sigma$ in non-IID split CIFAR-100 with $c_{\text{max}} = 40$, and partial participation with $20$ out of $100$ clients participating every round. The training duration was set to $t_{\text{max}} = 200$ rounds.}
%\resizebox{\columnwidth}{!}{
\begin{tabular}{ ccc }
 \toprule
\textbf{Framework} &  \textbf{Bitrate} & \textbf{Accuracy} \\
 \midrule
 \midrule
\texttt{FedPM-KLMS} (ours) & 0.074 $\pm$ 0.0001 & 0.488 $\pm$ 0.0013 \\
\texttt{FedPM-KLMS} (ours) & 0.048 $\pm$ 0.0001 & 0.484 $\pm$ 0.0013 \\
\texttt{FedPM-KLMS} (ours) & 0.012 $\pm$ 0.0001 & 0.480 $\pm$ 0.0013 \\
\texttt{QSGD-KLMS} (ours) & 0.072 $\pm$ 0.0001 & 0.428 $\pm$ 0.0013 \\
\texttt{QSGD-KLMS} (ours) & 0.040 $\pm$ 0.0001 & 0.424 $\pm$ 0.0013 \\
\texttt{QSGD-KLMS} (ours) & 0.017 $\pm$ 0.0001 & 0.419 $\pm$ 0.0013 \\
\texttt{SignSGD-KLMS} (ours) & 0.072 $\pm$ 0.0001 & 0.421 $\pm$ 0.0016 \\
\texttt{SignSGD-KLMS} (ours) & 0.044 $\pm$ 0.0001 & 0.419 $\pm$ 0.0016 \\
\texttt{SignSGD-KLMS} (ours) & 0.020 $\pm$ 0.0001 & 0.415 $\pm$ 0.0016 \\
\midrule
\texttt{FedPM}~\citep{isik2023sparse} & 0.980 $\pm$ 0.0001 & 0.488 $\pm$ 0.0012 \\
\texttt{QSGD}~\citep{alistarh2017qsgd} & 0.150 $\pm$ 0.0000 & 0.429 $\pm$ 0.0013  \\
\texttt{QSGD}~\citep{alistarh2017qsgd} & 0.082 $\pm$ 0.0000 & 0.424 $\pm$ 0.0013  \\
\texttt{DRIVE}~\citep{vargaftik2021drive} & 0.81 $\pm$ 0.0000 & 0.424 $\pm$ 0.0013 \\
\texttt{EDEN}~\citep{vargaftik2022eden} &  0.81 $\pm$ 0.0000 & 0.425 $\pm$ 0.0013 \\
\texttt{DP-REC}~\citep{triastcyn2021dp} & 1.00 $\pm$ 0.0001  & 0.424 $\pm$ 0.0014 \\
\texttt{DP-REC}~\citep{triastcyn2021dp} & 0.49 $\pm$ 0.0001  & 0.422 $\pm$ 0.0014 \\
\texttt{DP-REC}~\citep{triastcyn2021dp} & 0.27 $\pm$ 0.0001  & 0.412 $\pm$ 0.0014 \\
\texttt{DP-REC}~\citep{triastcyn2021dp} & 0.13 $\pm$ 0.0001  & 0.408 $\pm$ 0.0014 \\
\bottomrule
\end{tabular}%}
\label{tab:non_iid_40_class_cifar100_confidence}
\end{table}

\end{document}